\documentclass[11pt]{article}

\usepackage{amsmath,amssymb,amsfonts}
\usepackage{amsthm}
\usepackage{algorithm}
\usepackage{algpseudocode}
\usepackage{booktabs}
\usepackage{tabularx}
\usepackage{longtable}
\usepackage{array}
\usepackage{enumitem}
\usepackage{geometry}
\usepackage{graphicx}
\usepackage{float}
\usepackage{xurl}
\usepackage{caption}
\usepackage{hyperref}
\usepackage{xcolor}
\usepackage{listings}
\usepackage{tikz}
\usetikzlibrary{arrows.meta,calc,fit,positioning,shapes.geometric}

\setlist[itemize]{leftmargin=*}
\setlist[enumerate]{leftmargin=*}
\hypersetup{hidelinks}
\definecolor{deepblue}{RGB}{44,88,137}
\definecolor{bluegray}{RGB}{91,115,133}
\definecolor{slate}{RGB}{77,82,88}
\definecolor{sage}{RGB}{91,137,118}
\definecolor{amber}{RGB}{185,126,54}
\definecolor{rose}{RGB}{168,86,91}
\definecolor{palegray}{RGB}{248,249,250}
\definecolor{boxblue}{RGB}{220,232,248}
\definecolor{boxblueborder}{RGB}{90,130,195}
\definecolor{redbox}{RGB}{255,234,234}
\definecolor{redborder}{RGB}{205,60,60}
\definecolor{greenbox}{RGB}{232,248,232}
\definecolor{greenborder}{RGB}{55,155,55}
\definecolor{graybox}{RGB}{245,245,245}
\definecolor{grayborder}{RGB}{150,150,150}
\newcommand{\tablesource}[1]{\par\vspace{0.05em}{\raggedright\footnotesize\noindent Source: #1\par}\vspace{0.45em}}
\newcolumntype{Y}{>{\raggedright\arraybackslash}X}

\newtheorem{definition}{Definition}

\newtheorem{lemma}{Lemma}
\newtheorem{theorem}{Theorem}
\newtheorem{hypothesis}{Hypothesis}

\newcommand{\Profiles}{\mathcal{P}}
\newcommand{\Actions}{\mathcal{A}}
\newcommand{\Policies}{\Pi}
\newcommand{\Identities}{\mathcal{I}}
\newcommand{\Resources}{\mathcal{R}}
\newcommand{\Data}{\mathcal{D}}
\newcommand{\Memory}{\mathcal{M}}
\newcommand{\Threats}{\mathcal{Z}}
\newcommand{\join}{\sqcup}
\newcommand{\meet}{\sqcap}
\newcommand{\indicator}{\mathbf{1}}
\newcommand{\Low}{\mbox{\normalfont\textsc{Low}}}
\newcommand{\Medium}{\mbox{\normalfont\textsc{Medium}}}
\newcommand{\High}{\mbox{\normalfont\textsc{High}}}
\newcommand{\VeryHigh}{\mbox{\normalfont\textsc{VeryHigh}}}
\newcommand{\Decision}[1]{\mbox{\normalfont\textsc{#1}}}

\title{A Policy Algebra for Trust-Preserving Agentic AI Execution}
\author{%
\begin{minipage}{0.95\textwidth}
\centering
Bhaskar Tripathi\textsuperscript{1},\quad
Anurag Kumar\textsuperscript{1},\quad
Ramendra Kumar\textsuperscript{1},\quad
Bhavesh Gadhe\textsuperscript{2}
\\[0.9em]
{\small\textsuperscript{1}Volkswagen Digital Services,\quad
\textsuperscript{2}Scania CV AB}
\\[0.7em]
{\small Emails: \texttt{bhaskar.tripathi@gmail.com},
\texttt{anurag31296@gmail.com},\\
\texttt{karna.ramenk@gmail.com},
\texttt{bhaveshgadhe@gmail.com}}
\end{minipage}%
}
\date{}

\begin{document}
\maketitle

\begin{abstract}
Large language model--based agentic frameworks primarily optimize capability: whether an agent can reason, retrieve information, call tools, delegate work, and complete a goal. Enterprise execution requires a stronger property. A successful result is not reliable if it was produced through unauthorized data access, widened delegated authority, unapproved side effects, unrecoverable budget consumption, or incomplete evidence. This paper defines \emph{reliable capability} as a path property: an agent is reliably capable only when it completes a task through action events that remain admissible under identity, profile, tool, data, memory, budget, artifact, approval, and audit constraints. We propose a policy algebra that defines the reliability envelope within which agent capability may be exercised. Security profiles and runtime obligations compose through joins, intersections, budget narrowing, approval inheritance, and evidence accumulation; the resulting composition is both trust-preserving and the least restrictive state satisfying all governing inputs. The algebra also propagates restrictions across multi-agent calls and introduces cost-aware artifact materialization, which redirects open-ended execution toward a recoverable outcome as budget exposure grows. The evaluation is interpreted as a reliability--capability trade-off rather than a capability benchmark: the policy-algebra runtime intervenes on 94.8\% of policy-violating events while retaining an 86.9\% task-completion rate, eliminates the observed profile-monotonicity and zero-artifact-exhaustion violations, and increases audit completeness to 98.6\%. The method provides researchers and practitioners with formal correctness conditions, executable decision semantics, and trace evidence for building agents that are not only capable, but reliably capable.
\end{abstract}


\clearpage

\section{Introduction}

Agentic AI systems change the security problem of machine learning applications. A chatbot produces text; an agent can act. It may search documents, retrieve confidential knowledge, call Representational State Transfer (REST) APIs, use a Model Context Protocol (MCP) tool, update memory, create a ticket, send an email, trigger a workflow, write to a database, or delegate part of a task to another agent. Recent language-agent work shows this shift from text generation to interactive action through web interaction, tool use, memory, planning, and environment feedback \cite{webshop,react,coala,generative-agents,webarena,agentbench-llm-agents}. These capabilities make agents attractive for enterprise automation because they connect reasoning to execution. The same connection creates a new class of reliability and security requirements. An unsafe answer is no longer the only failure mode; an unsafe action, delegation, memory write, tool call, or service publication can affect enterprise systems directly \cite{he-security-ai-agents}.

The central intuition is that agentic security is a path property. Language-agent methods such as ReAct interleave reasoning traces, actions, and observations, while cognitive-agent architectures explicitly model memory, action spaces, and decision processes \cite{react,coala}. A final answer is reliable only when the sequence of actions that produced it is itself admissible. Capability is existential: it asks whether at least one execution can reach the goal. Reliability is universal over the events of the selected execution: every transition that converts reasoning into action must satisfy the governing invariants. The scientific objective is therefore not to reduce capability, but to maximize useful autonomy inside a reliability envelope.

In this paper, \emph{reliability} is used in a security and economic sense: an agentic action is reliable only when the runtime can explain who initiated it, which authority was used, which policy allowed it, which data or memory it touched, which human approval was required, which budget it consumed, which durable artifact or output state it produced, and which audit evidence remains after execution. The focus is therefore the transition between reasoning and action. This transition is where an LLM suggestion becomes an API call, a tool invocation, a sub-agent request, a memory update, or a published service response. If that transition is not governed, even a benign user request may produce unauthorized side effects or consume capital without delivering a recoverable result.

Existing standards and frameworks provide important foundations. OWASP AISVS defines AI security verification categories that include input validation, access control, memory security, agentic action security, monitoring, privacy, and human oversight \cite{owasp-aisvs}. NIST AI RMF organizes enterprise AI risk management through Govern, Map, Measure, and Manage \cite{nist-airmf}. MAESTRO provides a layered threat model for agentic AI systems \cite{maestro}. AgenticSecurity.info aggregates AISVS controls, NIST mappings, threat catalogs, mitigation catalogs, component taxonomies, architecture patterns, and threat-modeler outputs \cite{agenticsecurity}. These resources help engineers identify what can go wrong and which controls may be relevant.

However, assessment artifacts do not by themselves define execution semantics. Existing benchmarks and environments evaluate interactive task completion, web navigation, and agent decision-making, but their primary metrics are functional correctness or task success rather than enterprise authorization, auditability, or profile-preserving execution \cite{webshop,agentbench-llm-agents,webarena}. A platform still needs to decide whether an agent may be created, whether a user can attach a connector, whether a tool call should be denied or escalated, whether a memory source is visible under a caller's role, whether a high-risk workflow requires human approval, whether a service published under one profile can be invoked by a caller under another profile, and whether a sub-agent can reduce the effective security posture of a call chain. These decisions must be made consistently at runtime, not only during design review.

The research literature also shows that agentic security is not a single problem. Public testbeds such as AgentDojo and Agent Security Bench evaluate prompt injection, tool abuse, memory poisoning, and mixed attacks \cite{agentdojo,asb}. Runtime defenses such as ClawGuard and DRIFT intercept or isolate unsafe tool-call behavior \cite{clawguard,drift}. NeuroTaint studies semantic information flow and cross-session persistence in agents \cite{neurotaint}. Governance-oriented work such as AGENTSAFE and MI9 moves toward design-time, runtime, and audit controls \cite{agentsafe,mi9}. MCP security analyses identify protocol-specific risks such as capability attestation gaps, origin-authentication gaps, and implicit trust propagation \cite{mcp-security}. These works are useful comparative references, but they are not treated as standards in this paper. Enterprises still require a compositional method that can place such controls inside a single runtime policy model.

The paper addresses this gap through a policy algebra for trust-preserving agentic execution. The need for an execution-level model follows from prior work showing that agents can generate actions, call APIs or tools, maintain memory, and navigate large action spaces \cite{react,coala,toolchain-star,reverse-chain}. The algebra treats security profiles as ordered policy objects. It resolves enterprise policy through a cascade from platform to team, user, agent, and publication scopes. It authorizes tools by intersecting role, agent, service-exposure, layer, data, budget, artifact-state, and approval predicates. It extends naturally to multi-agent execution by using profile joins, tool-set intersections, budget narrowing, memory-scope narrowing, artifact materialization, and human-approval propagation. The resulting runtime does not ask only whether an agent has a tool. It asks whether this caller, this agent, this service, this data class, this profile, this budget, this artifact state, and this approval state jointly permit this action.

The motivating enterprise setting is practical. A platform may provide an agent framework, a runtime execution layer, and one or more applications. Users and teams create agents, attach knowledge sources, configure tools, expose services through MCP and REST interfaces, and invoke published agents from other applications. WebShop, AgentBench, and WebArena show that language agents can operate over web-like, tool-rich, and multi-step interactive environments, while Security of AI Agents shows that these capabilities expose confidentiality, integrity, and availability risks when actions are insufficiently constrained \cite{webshop,agentbench-llm-agents,webarena,he-security-ai-agents}. In such an environment, local controls are insufficient. A profile selected in the UI must be reflected in backend enforcement. A permission gate must see user roles, agent allowlists, service-publication metadata, data classification, budget state, and artifact state. A service approval must freeze the security context under which the service was reviewed. A sub-agent cannot be allowed to launder a high-security call into a lower-security execution path.

Consider an internal audit agent that can read invoices, compare purchase orders, query an ERP system, create exceptions, and request clarifications from another specialist agent. A conventional access-control layer may verify that the user can open the audit application, and a prompt guardrail may detect some malicious text. Prior work on protection systems and role-based access control motivates the access-control part of this requirement, while prompt-injection and jailbreak studies show why prompt-level defenses alone are insufficient for agentic execution \cite{saltzer-schroeder,ferraiolo-rbac,llmfuzzer,struq,he-security-ai-agents}. The audit run also needs to know which invoice classes the user may inspect, whether the ERP query is read-only, whether the exception-creation tool is destructive, whether the specialist agent inherits the initiating user's constraints, whether a high-value exception requires human approval, and whether the final answer exposes confidential supplier information. The security property is therefore not attached to one prompt or one API endpoint. It is attached to a chain of decisions.

This chain perspective motivates four requirements. First, the runtime must treat policy as compositional. RBAC, profiles, tool allowlists, data classifications, publication metadata, artifact state, and human approvals must combine into one executable decision. Second, the runtime must be monotone across delegation. When a caller invokes another agent, the child context must not be less restrictive than the parent context. Third, the runtime must enforce bounded loss for unattended execution, so budget exhaustion without a durable artifact is treated as a policy failure rather than merely an expensive run. Fourth, the runtime must produce evidence. These requirements reflect the fact that agent traces are not just text histories: they contain reasoning, tool calls, observations, memory accesses, and environment effects \cite{react,coala,generative-agents,he-security-ai-agents}. If the system denies a tool call, redirects to artifact materialization, escalates to a human, redacts output, or freezes a publication profile, the decision should be visible in a trace that supports review and incident response.

The proposed policy algebra does not require every platform to implement the same user interface or the same control mechanism. Instead, it defines the algebraic invariants that any implementation should preserve. Profiles form an ordered policy space. Delegation uses joins rather than overwrites. Tool authorization is a conjunction, not a single allowlist check. Memory access is governed by classification, scope, purpose, and retention. Publication creates a versioned security snapshot. Cost-aware materialization creates a recoverable output obligation under budget pressure. Audit completeness is part of the feasibility of an execution, not an optional afterthought.

The research question is therefore: \emph{How can an enterprise agent platform preserve useful agent capability while ensuring that every executed action remains authorized, non-amplifying under delegation, economically recoverable, and auditable?} The answer is obtained by compiling heterogeneous guidance into executable predicates and composing those predicates into a reliability envelope that constrains, but does not otherwise replace, the agent's task reasoning. This framing is consistent with prior evidence that language agents execute multi-step decisions in external environments, but it shifts the evaluation target from task success alone to the joint achievement of task success and admissible, auditable execution \cite{agentbench-llm-agents,webarena,he-security-ai-agents}.

This question differs from conventional access control in two ways. First, the protected object is not only a resource such as a file or API. The protected object is often a transition: a model-generated proposal becomes an action under a particular identity, profile, memory state, artifact state, and call chain. Second, the authorization state is not static. Each step can change the remaining budget, memory scope, artifact state, approval state, and available evidence. A secure runtime must therefore evaluate policies repeatedly as the execution unfolds.

The question also differs from ordinary workflow governance. In a traditional workflow, designers typically know the sequence of steps in advance. Classical planning assumes explicit models of actions and goals \cite{wilkins-planning}. In an agentic workflow, however, the reasoner may propose the next tool call dynamically, as in reasoning-and-acting methods and multi-API planning methods \cite{react,toolchain-star,reverse-chain}. The platform is not assumed to pre-enumerate every action path. Instead, every proposed path must pass through the same algebraic constraints. This makes the method suitable for both scripted workflows and adaptive agents.

Finally, the paper distinguishes between assessment and enforcement. A threat model can say that a workflow has a tool-misuse risk. An enforcement model must decide whether the specific tool call should be allowed, denied, sandboxed, approved by a human, or logged with stronger evidence. This distinction is important because agent-security studies identify concrete vulnerabilities in sessions, model interaction, agent programs, prompt injection, and tool execution, but the runtime still needs an enforceable decision procedure for each proposed action \cite{he-security-ai-agents,llmfuzzer,struq}. The method connects these levels by mapping threat-model artifacts to policy predicates. The resulting runtime can use assessment outputs, but it does not stop at assessment.

The paper makes four contributions.
\begin{enumerate}
    \item A definition of \emph{reliable capability} that distinguishes reaching a goal from reaching it through an admissible action path, and formulates agent autonomy as utility maximization inside a reliability envelope.
    \item A policy algebra that composes enterprise profiles, RBAC, tools, memory, publication, budget, approval, and evidence obligations. The composition is trust-preserving and is the least restrictive policy state satisfying all governing inputs.
    \item Runtime semantics for non-amplifying multi-agent delegation and cost-aware artifact materialization, including profile joins, authority intersections, budget narrowing, approval inheritance, and recoverable-output obligations.
    \item A design-science path from security artifacts to executable predicates, a reference runtime, conformance measures, capability--reliability trade-off analysis, and policy repair through regression testing.
\end{enumerate}

Figure~\ref{fig:action-path-reliability} summarizes the proposed shift from output-only assessment to action-path reliability.

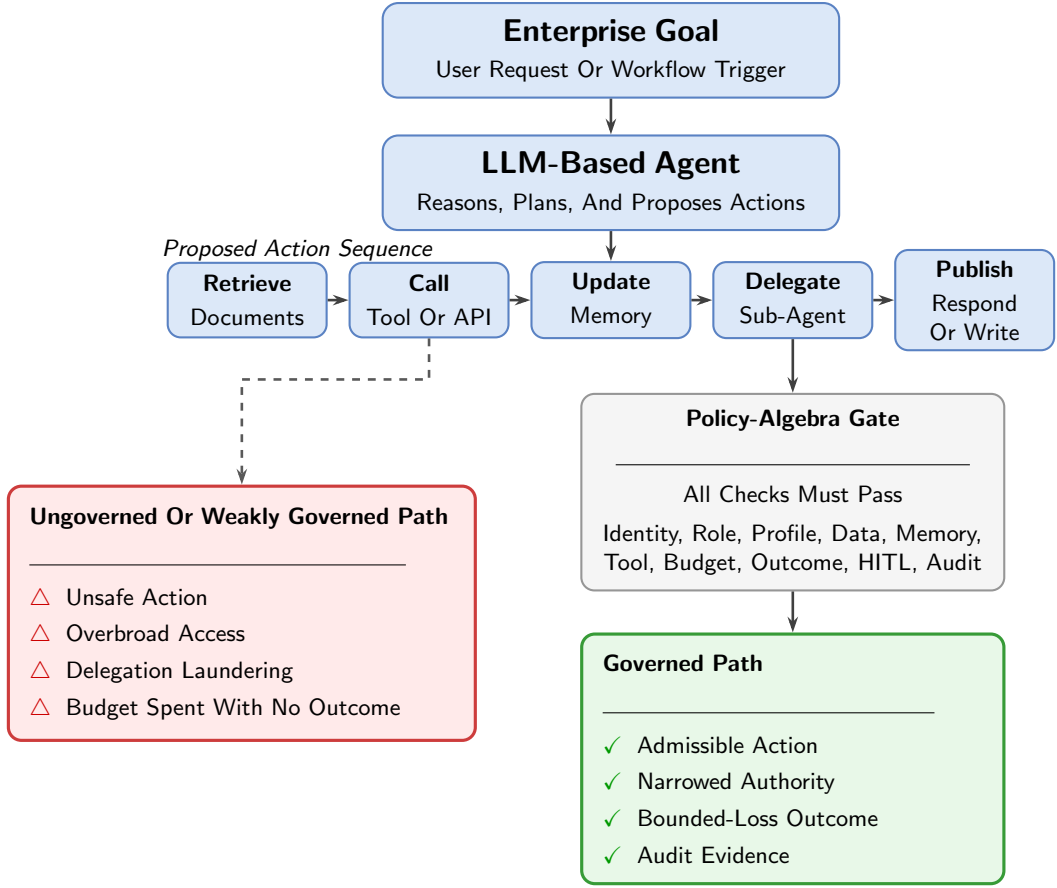
\begin{figure}[!t]
\centering
\resizebox{0.84\textwidth}{!}{%
\begin{tikzpicture}[
  font=\sffamily\footnotesize,
  mainbox/.style={
    rectangle, rounded corners=6pt,
    draw=boxblueborder, fill=boxblue, line width=0.8pt,
    text width=5.7cm, align=center,
    inner sep=6pt, minimum height=0.9cm
  },
  actionbox/.style={
    rectangle, rounded corners=5pt,
    draw=boxblueborder, fill=boxblue, line width=0.8pt,
    text width=1.85cm, align=center,
    inner sep=4pt, minimum height=1.0cm
  },
  policybox/.style={
    rectangle, rounded corners=5pt,
    draw=grayborder, fill=graybox, line width=0.8pt,
    text width=5.25cm, align=center,
    inner sep=6pt
  },
  redoutbox/.style={
    rectangle, rounded corners=5pt,
    draw=redborder, fill=redbox, line width=1.2pt,
    text width=5.7cm, align=left,
    inner sep=8pt
  },
  greenoutbox/.style={
    rectangle, rounded corners=5pt,
    draw=greenborder, fill=greenbox, line width=1.2pt,
    text width=5.1cm, align=left,
    inner sep=8pt
  },
  myarrow/.style={-{Stealth[length=6pt, width=4pt]}, line width=1pt, color=black!75},
  dasharrow/.style={-{Stealth[length=6pt, width=4pt]}, line width=1pt, dashed, color=black!65}
]

\node[actionbox] (update) at (0, -3.35cm) {
  \textbf{Update}\\[2pt]\footnotesize Memory
};
\node[actionbox, left=0.28cm of update]  (call)     { \textbf{Call}\\[2pt]\footnotesize Tool Or API };
\node[actionbox, left=0.28cm of call]    (retrieve) { \textbf{Retrieve}\\[2pt]\footnotesize Documents };
\node[actionbox, right=0.28cm of update] (delegate) { \textbf{Delegate}\\[2pt]\footnotesize Sub-Agent };
\node[actionbox, right=0.28cm of delegate] (publish) { \textbf{Publish}\\[2pt]\footnotesize Respond Or Write };

\node[mainbox] (goal) at (0, 0) {
  \textbf{\large Enterprise Goal}\\[3pt]
  \footnotesize User Request Or Workflow Trigger
};

\node[mainbox, below=0.48cm of goal] (llm) {
  \textbf{\large LLM-Based Agent}\\[3pt]
  \footnotesize Reasons, Plans, And Proposes Actions
};

\draw[myarrow] (goal) -- (llm);
\draw[myarrow] (llm.south) -- (update.north);

\node[anchor=west, font=\sffamily\footnotesize\itshape]
  at ([xshift=-0.2cm]retrieve.north west) [yshift=0.15cm] {Proposed Action Sequence};

\draw[myarrow] (retrieve) -- (call);
\draw[myarrow] (call)     -- (update);
\draw[myarrow] (update)   -- (delegate);
\draw[myarrow] (delegate) -- (publish);

\node[policybox, below=0.72cm of delegate] (policy) {
  \textbf{Policy-Algebra Gate}\\[4pt]
  \rule{4.75cm}{0.4pt}\\[3pt]
  {\footnotesize All Checks Must Pass}\\[4pt]
  {\footnotesize Identity, Role, Profile, Data, Memory,}\\
  {\footnotesize Tool, Budget, Outcome, HITL, Audit}
};
\draw[myarrow] (delegate.south) -- (policy.north);

\node[redoutbox] (ungov) at (-4.95cm, -7.55cm) {
  \textbf{Ungoverned Or Weakly Governed Path}\\[5pt]
  \rule{5.05cm}{0.4pt}\\[4pt]
  {\color{red!80!black}$\triangle$}\; \footnotesize Unsafe Action\\[3pt]
  {\color{red!80!black}$\triangle$}\; \footnotesize Overbroad Access\\[3pt]
  {\color{red!80!black}$\triangle$}\; \footnotesize Delegation Laundering\\[3pt]
  {\color{red!80!black}$\triangle$}\; \footnotesize Budget Spent With No Outcome
};

\draw[dasharrow] (call.south) -- ++(0,-0.55) -| (ungov.north);

\node[greenoutbox, below=0.55cm of policy] (gov) {
  \textbf{Governed Path}\\[5pt]
  \rule{4.45cm}{0.4pt}\\[4pt]
  {\color{green!60!black}$\checkmark$}\; \footnotesize Admissible Action\\[3pt]
  {\color{green!60!black}$\checkmark$}\; \footnotesize Narrowed Authority\\[3pt]
  {\color{green!60!black}$\checkmark$}\; \footnotesize Bounded-Loss Outcome\\[3pt]
  {\color{green!60!black}$\checkmark$}\; \footnotesize Audit Evidence
};
\draw[myarrow] (policy.south) -- (gov.north);
\end{tikzpicture}%
}
\caption{From output safety to action-path reliability. A chatbot-style assessment can focus on the final answer, but an agentic system must govern each reasoning-to-action transition. The proposed policy algebra admits actions only when identity, role, profile, data, memory, tool, budget, outcome, human-approval, and audit predicates jointly hold, thereby preventing unsafe actions and zero-outcome budget exhaustion.}
\label{fig:action-path-reliability}
\end{figure}

To the best of our knowledge, this is an early implementation-oriented formalization that connects standards aggregation, enterprise profile governance, multi-agent trust propagation, MCP, REST publication control, RAG and memory security, and continuous verification inside one runtime semantics. The aim is not to replace existing standards or defenses. The aim is to make them executable in an enterprise agent platform.

\section{Background and Related Work}

The related literature is best read as identifying risk surfaces rather than specifying a single enforcement semantics. This section separates those two roles.

\subsection{Agentic AI and Multi-Agent Systems}

Agency matters because a model output can become a state-changing operation.

An agentic system contains one or more reasoning modules that can decompose tasks, call tools, retrieve data, update memory, and delegate work. This view extends the classical agent and multi-agent systems literature, where agents are treated as autonomous entities that perceive, decide, act, and interact with other agents in a shared environment \cite{shoham-aop,wooldridge-jennings-agents,jennings-roadmap,stone-veloso-survey}. Agents may be sequential, hierarchical, reactive, collaborative, or knowledge-intensive. Once agents receive access to tools and memory, they become actors in an operational environment. This changes the security target from content filtering to governed action.

Multi-agent systems add trust-chain complexity. A caller may invoke a sub-agent with a different tool set, publication profile, memory scope, or service exposure. Without explicit propagation rules, delegation can widen authority. Recent work on multi-agent trust describes the tension between collaboration and over-exposure \cite{trust-paradox}. This paper operationalizes that concern through monotone profile joins and narrowing of tools, budget, and memory.

\subsection{Standards, Threat Models, and Aggregators}

Standards and threat models define what should be controlled; the runtime still has to decide how those controls compose.

AISVS, NIST AI RMF, and MAESTRO provide complementary perspectives. AISVS defines verifiable security requirements \cite{owasp-aisvs}. NIST AI RMF organizes institutional risk management \cite{nist-airmf}. MAESTRO locates agentic risks across layers such as models, data operations, frameworks, tools, deployment, observability, and agent ecosystems \cite{maestro}. AgenticSecurity.info combines these ideas with threat catalogs, mitigation catalogs, component taxonomies, architecture patterns, and threat-modeler outputs \cite{agenticsecurity}. These artifacts serve as a knowledge substrate.

\subsection{Runtime Defenses, Testbeds, and Governance}

Runtime defenses and testbeds supply useful cases and controls, but they do not by themselves define the enterprise policy state.

Public testbeds such as AgentDojo and Agent Security Bench provide examples of adversarial prompt, tool, and memory conditions \cite{agentdojo,asb}. ClawGuard and DRIFT focus on runtime defenses for indirect prompt injection and tool-boundary enforcement \cite{clawguard,drift}. NeuroTaint frames information flow for agents that transform and persist information across sessions \cite{neurotaint}. AGENTSAFE and MI9 study broader governance and runtime assurance \cite{agentsafe,mi9}. MCP security work identifies risks in tool-integrated protocol settings \cite{mcp-security}. The method is complementary, but its evaluation standard is the proposed policy algebra, control taxonomy, and conformance protocol rather than any single public testbed.

\subsection{Agentic Identity and Access Management}

Identity is necessary but not sufficient; after authentication, execution authority must still be narrowed by profile, memory, tool, and publication constraints.

Agentic IAM work proposes decentralized identifiers, verifiable credentials, zero-knowledge proofs, agent naming services, and session enforcement for agent authentication and access control \cite{agentic-iam}. Authentication and base authorization are assumed to be established. The focus is the execution layer: how policies constrain tools, memory, publication, budget, artifact production, approvals, and delegation after identity is known.

\subsection{Cost-Aware Execution and Bounded-Loss Design}

Cost is also part of the execution context in unattended agentic systems. A workflow that consumes all allocated budget while producing no durable artifact may satisfy a simple spending cap, but it still fails as a governed execution because the platform has no recoverable output to return, inspect, or resume. This failure mode is especially important for asynchronous enterprise agents, where the user is not continuously supervising each reasoning step.

The relevant risk-management literature emphasizes that systems operating under uncertainty should first avoid ruin-like states and unmanaged downside before optimizing expected gain \cite{taleb-risk-mistakes,taleb-fat-tails}. Taleb and Douady formalize antifragility in terms of payoff behavior under stress and volatility, distinguishing systems that are merely robust from systems whose payoff structure benefits from bounded perturbations \cite{taleb-douady-antifragility}. In this paper, the corresponding runtime question is narrower and operational: can the platform prevent an execution trace whose cost is fully realized while its artifact state remains empty?

The answer is a materialization predicate inside the same policy algebra used for tools, memory, budget, approval, and audit. The runtime allows research and planning while cost exposure remains acceptable. Once cost consumption becomes material and no artifact exists, the admissible action set is constrained toward producing a durable draft, checkpoint, file, record, or partial result. This makes bounded-loss behavior a platform-enforced property rather than a voluntary self-regulation behavior expected from the model.

\begingroup
\footnotesize
\setlength{\tabcolsep}{3pt}
\renewcommand{\arraystretch}{1.08}
\setlength{\LTpre}{0.8em}
\setlength{\LTpost}{0pt}
\begin{longtable}{@{}>{\raggedright\arraybackslash}p{0.05\textwidth}
>{\raggedright\arraybackslash}p{0.22\textwidth}
>{\raggedright\arraybackslash}p{0.33\textwidth}
>{\raggedright\arraybackslash}p{0.28\textwidth}@{}}
\caption{Related work and gap addressed by this paper.}
\label{tab:related}\\
\toprule
No. & Work or standard & Main focus and contribution & Gap addressed by this paper \\
\midrule
\endfirsthead
\toprule
No. & Work or standard & Main focus and contribution & Gap addressed by this paper \\
\midrule
\endhead
1 & OWASP AISVS \cite{owasp-aisvs} & AI security verification through testable control categories and levels & Action-level enforcement inside an agent loop \\
2 & NIST AI RMF and GenAI profile \cite{nist-airmf,nist-genai} & AI risk governance through the Govern, Map, Measure, and Manage cycle & Runtime semantics for the risk cycle \\
3 & MAESTRO \cite{maestro} & Agentic threat modeling through layered analysis of agentic systems & Translation from layer finding to policy gate \\
4 & MITRE ATLAS \cite{mitre-atlas} & AI adversary tactics and techniques for AI systems & Runtime use of tactics as control requirements \\
5 & Agentic Security Hub \cite{agenticsecurity} & Aggregated standards, threats, components, architectures, and mappings & Compilation into executable runtime constraints \\
6 & Agent and multi-agent systems foundations \cite{shoham-aop,wooldridge-jennings-agents,jennings-roadmap,stone-veloso-survey} & Autonomy, interaction, coordination, and multi-agent system design & Runtime security semantics for tool-using agentic systems \\
7 & Public agent-security testbeds \cite{agentdojo,asb} & Comparative workloads for prompt, tool, memory, and mixed attack cases & Optional comparison; conformance is defined by the proposed runtime predicates \\
8 & Agent Security Bench \cite{asb} & Attack and defense testbed for prompt, tool, memory, backdoor, and mixed attacks & Mapping attacks to control predicates and traces \\
9 & ClawGuard \cite{clawguard} & Tool-boundary defense through deterministic rule enforcement at tool calls & Integration with profile, data, budget, and audit constraints \\
10 & DRIFT \cite{drift} & Dynamic injection defense through rules and memory-stream isolation & Integration with enterprise policy cascade \\
11 & NeuroTaint \cite{neurotaint} & Agent information-flow analysis using trace and memory-based taint tracking & Formal memory/RAG predicates and assurance signals \\
12 & AGENTSAFE \cite{agentsafe} & Governance framework for lifecycle risk-to-control mapping & Policy algebra and executable profile cascade \\
13 & MI9 \cite{mi9} & Runtime governance using risk index, telemetry, and authorization monitoring & Trust algebra and publication-aware action control \\
14 & Agentic IAM \cite{agentic-iam} & Agent identity through verifiable identity and access mechanisms & Post-authentication execution control \\
15 & MCP security analysis \cite{mcp-security} & MCP protocol risk, including capability attestation and trust propagation & Profile-capped MCP, REST publication governance \\
16 & Trust Paradox \cite{trust-paradox} & Multi-agent trust and the exposure risk created by collaboration & Monotonic trust rule and call-chain constraints \\
17 & Quantitative risk and antifragility \cite{taleb-risk-mistakes,taleb-fat-tails,taleb-douady-antifragility} & Downside control, fat-tail risk, and payoff behavior under uncertainty & Cost-aware artifact materialization for bounded-loss unattended execution \\
\bottomrule
\end{longtable}
\tablesource{Authors' synthesis from cited studies}
\endgroup

\section{Problem Formulation: From Capability to Reliable Capability}
\label{sec:problem}

The basic distinction is between completing a task and completing it through an admissible path. A runtime that blocks every action may be safe but is not useful; a runtime that completes every task by any available means may be capable but is not trustworthy. Reliable capability requires both goal achievement and path admissibility.

Capability alone can hide risk: an agent may answer correctly after reading data it should not access, call a tool that the user could not call directly, delegate to a weaker profile, write memory that changes future behavior, consume the full budget without producing a durable artifact, or complete a workflow without an audit trail. Reliability is therefore formulated as a property of the execution path, not only of the final answer.

The execution trace is therefore the primitive object of analysis.

Let an execution be a trace
\[
    \xi=(s_0,a_0,o_0,s_1,a_1,o_1,\ldots,s_T,a_T,o_T),
\]
where $s_t$ is the runtime state, $a_t$ is an action chosen by the agent runtime, and $o_t$ is the observation returned by a model, tool, memory system, service, or environment. An action may be an internal reasoning step, a tool call, a service invocation, a memory operation, a data retrieval, a publication action, or a delegation to another agent. The question is not only whether $a_t$ helps complete the task, but whether it is feasible under the governing security context.

Let $\Xi(g,\zeta)$ be the set of traces that agent or service $g$ can generate for task $\zeta$.

\begin{definition}[Agent capability]
Agent $g$ is capable of task $\zeta$ if at least one executable trace reaches the task goal:
\begin{equation}
    \mathsf{Capable}(g,\zeta)
    \iff
    \exists \xi\in\Xi(g,\zeta):
    \mathsf{GoalReached}(\xi,\zeta)=1.
    \label{eq:capability}
\end{equation}
Capability is therefore an existential property. It does not constrain the path by which the goal is reached.
\end{definition}

The security context at time $t$ is modeled as
\[
    c_t=(u,g,p_t,R_t,K_t,B_t,H_t,L_t,O_t,A_t),
\]
where $u$ is the initiating identity, $g$ the agent or service, $p_t$ the effective security profile, $R_t$ the role and resource state, $K_t$ the memory and knowledge scope, $B_t$ the remaining budget, $H_t$ the human-approval state, $L_t$ the call-chain lineage, $O_t$ the durable artifact or checkpoint state, and $A_t$ the audit state. The symbol $\varpi_t$ denotes the policy state induced by this context.

\begin{definition}[Action event]
An action event is the tuple
\[
    \eta_t=(u,g,s_t,a_t,o_t,c_t,\varpi_t).
\]
It binds the actor, agent or service, selected action, returned observation, runtime context, and effective policy state for a single reasoning-to-action transition.
\end{definition}

\begin{definition}[Reliable agent execution]
An execution $\xi$ is reliable under policy state sequence $\varpi_{0:T}$ if every action event $\eta_t$ satisfies
\begin{equation}
    \mathsf{Id}_t
    \land \mathsf{Role}_t
    \land \mathsf{Prof}_t
    \land \mathsf{Data}_t
    \land \mathsf{Mem}_t
    \land \mathsf{Tool}_t
    \land \mathsf{Bud}_t
    \land \mathsf{Art}_t
    \land \mathsf{Hitl}_t
    \land \mathsf{Aud}_t
    =1,
    \label{eq:reliable-predicate}
\end{equation}
where the predicates denote identity validity, role authorization, profile compliance, data compliance, memory/RAG compliance, tool or service compliance, budget compliance, artifact materialization compliance, human-approval compliance, and audit completeness, respectively. For every delegation edge $i\rightarrow j$, the delegated context must also satisfy
\begin{equation}
    p_i\preceq p_{i\rightarrow j},\qquad
    T_{i\rightarrow j}\subseteq T_i,\qquad
    \Theta_{i\rightarrow j}\subseteq \Theta_i,\qquad
    B_{i\rightarrow j}\le B_i.
    \label{eq:delegation-reliable}
\end{equation}
where $p_{i\rightarrow j}$, $T_{i\rightarrow j}$, $\Theta_{i\rightarrow j}$, and $B_{i\rightarrow j}$ are the profile, tool set, memory scope, and budget assigned to the delegated call from $i$ to $j$. Thus delegation may tighten the profile, narrow tools, narrow memory scope, and reduce budget, but it may not widen the authority inherited from the caller.
\end{definition}

\begin{definition}[Reliable capability]
Agent $g$ is reliably capable of task $\zeta$ under policy sequence $\varpi_{0:T}$ if it can reach the goal through at least one reliable trace:
\begin{equation}
\begin{aligned}
    \mathsf{ReliablyCapable}(g,\zeta,\varpi_{0:T})
    \iff{}& \exists \xi\in\Xi(g,\zeta):
    \mathsf{GoalReached}(\xi,\zeta)=1\\
    &\land\ \mathsf{Reliable}(\xi,\varpi_{0:T})=1.
\end{aligned}
\label{eq:reliable-capability}
\end{equation}
The goal condition is existential over traces, while reliability is universal over the action events in the selected trace. Reliable capability is consequently stronger than raw capability, but it does not require the runtime to reject actions that already satisfy every applicable constraint.
\end{definition}

Equation~\eqref{eq:reliable-predicate} states reliability as a conjunction of independent Boolean obligations. A single failed predicate makes the action unreliable. Equation~\eqref{eq:delegation-reliable} adds the multi-agent condition needed to prevent profile laundering and authority widening across call chains.

Let $\mathsf{Narrow}(a_t,c_t)=1$ when $a_t$ is not a delegation or when the delegation relation satisfies \eqref{eq:delegation-reliable}. The principal failure modes are indicator functions over these predicates:
\begin{equation}
\begin{alignedat}{2}
F_{\mathrm{id}}(t) &= 1-\mathsf{Id}_t, &\quad& \text{invalid or missing actor identity},\\
F_{\mathrm{role}}(t) &= 1-\mathsf{Role}_t, &\quad& \text{role does not authorize the action},\\
F_{\mathrm{profile}}(t) &= 1-\mathsf{Prof}_t, &\quad& \text{action violates the effective profile},\\
F_{\mathrm{data}}(t) &= 1-\mathsf{Data}_t, &\quad& \text{data class or scope is not allowed},\\
F_{\mathrm{mem}}(t) &= 1-\mathsf{Mem}_t, &\quad& \text{memory read, write, or retrieval is not allowed},\\
F_{\mathrm{tool}}(t) &= 1-\mathsf{Tool}_t, &\quad& \text{tool or service call is not allowed},\\
F_{\mathrm{budget}}(t) &= 1-\mathsf{Bud}_t, &\quad& \text{budget or rate limit is exceeded},\\
F_{\mathrm{artifact}}(t) &= 1-\mathsf{Art}_t, &\quad& \text{cost exposure is high but no durable artifact exists},\\
F_{\mathrm{trust}}(t) &= 1-\mathsf{Narrow}(a_t,c_t), &\quad& \text{delegation widens profile, tools, memory, or budget},\\
F_{\mathrm{hitl}}(t) &= 1-\mathsf{Hitl}_t, &\quad& \text{required human approval is absent},\\
F_{\mathrm{audit}}(t) &= 1-\mathsf{Aud}_t, &\quad& \text{required evidence is missing}.
\end{alignedat}
\label{eq:failure-modes}
\end{equation}

The aggregate failure score is
\begin{equation}
\begin{aligned}
    F(\xi)
    ={}& \sum_{t=0}^{T}\Big(
    w_1F_{\mathrm{id}}(t)+
    w_2F_{\mathrm{role}}(t)+
    w_3F_{\mathrm{profile}}(t)+
    w_4F_{\mathrm{data}}(t)\\
    &\quad+
    w_5F_{\mathrm{mem}}(t)+
    w_6F_{\mathrm{tool}}(t)+
    w_7F_{\mathrm{budget}}(t)+
    w_8F_{\mathrm{artifact}}(t)+
    w_9F_{\mathrm{trust}}(t)\\
    &\quad+
    w_{10}F_{\mathrm{hitl}}(t)+
    w_{11}F_{\mathrm{audit}}(t)
    \Big).
\end{aligned}
    \label{eq:failure}
\end{equation}
where $w_i\ge0$ is the relative weight assigned to the $i$th failure class. A hard-policy deployment sets the admissible region by requiring each failure term to be zero; a risk-scoring deployment may use the weighted sum only after hard feasibility has been checked. Reliable execution requires $F(\xi)=0$ for hard security constraints and requires \eqref{eq:delegation-reliable} for all delegation edges. This formulation gives a precise meaning to the failure cases considered in the rest of the paper: wrong authority is an identity or role failure, overbroad tools are tool failures or delegation violations, data leakage is a data failure, memory contamination is a memory failure, profile laundering is a delegation-profile violation, unrecoverable cost without output is an artifact-materialization failure, unapproved external action is a HITL failure, and missing traceability is an audit failure. In deployments that permit risk scoring, $F(\xi)$ can also serve as one input to the residual-risk model defined in the preliminaries.

\subsection{Conditional Correctness and Context Perception}

The policy algebra decides over a structured execution context, but some context fields may be observed rather than intrinsically known. Let $c_t$ denote the correct context and $\widehat{c}_t$ the context presented to the policy gate after tool metadata lookup, data classification, provenance checking, or semantic classification. End-to-end unsafe execution can arise either because the context was represented incorrectly or because the gate made an incorrect decision despite a correct context. By event decomposition,
\begin{equation}
\begin{aligned}
\Pr(\mathsf{UnsafeAllowed})
\le{}& \Pr(\widehat{c}_t\ne c_t)\\
&+\Pr(\mathsf{UnsafeAllowed}\mid \widehat{c}_t=c_t).
\end{aligned}
\label{eq:conditional-correctness}
\end{equation}
The algebra and its conformance tests target the second term: conditional on correct identity, policy, tool, data, memory, budget, artifact, and approval facts, the decision procedure should preserve the stated invariants. The first term belongs to the observation and metadata boundary. It includes misclassified documents, ambiguous tool intent, stale service metadata, incomplete command classification, and missing action-risk tags. This separation prevents a deterministic policy guarantee from being confused with the accuracy of the systems that supply its facts, and it provides a direct interpretation for the residual failure modes reported in the evaluation.

\section{System Model}

Profiles play the role of ordered risk regimes: moving upward in the order can only increase assurance obligations.

Let $\Profiles=\{\Low,\Medium,\High,\VeryHigh\}$ be a finite ordered set with
\[
    \Low \preceq \Medium \preceq \High \preceq \VeryHigh.
\]
The order denotes increasing assurance and restriction. The join $p_i\join p_j$ returns the more restrictive profile, and the meet $p_i\meet p_j$ returns the less restrictive profile. Each profile is a policy object
\[
    \Gamma(p)=(M_p,T_p,C_p,\Theta_p,\tau_p,B_p,Q_p,H_p,E_p,S_p,A_p),
\]
where $M_p$ is the model set, $T_p$ the tool set, $C_p$ the admissible data classes, $\Theta_p$ the memory-scope relation, $\tau_p$ retention limits, $B_p$ budget, $Q_p$ artifact-materialization policy, $H_p$ approval predicate, $E_p$ external exposure policy, $S_p$ schema/determinism policy, and $A_p$ audit-depth requirement.

The profile order represents assurance obligations, not an independent grant of business privilege or data clearance. A higher profile may require stronger validation, isolation, approval, and evidence, but it cannot by itself authorize a tool or resource that the initiating identity, agent configuration, data policy, or publication state does not authorize. Effective authority is always obtained through the intersections defined below. This separation prevents stronger assurance from being misread as broader entitlement.

Let $\Identities$ be the set of users, teams, service accounts, applications, MCP clients, workflows, agents, and published services. Let $\Resources$ be the set of tools, connectors, knowledge sources, memory stores, models, APIs, and service endpoints. The runtime maps $(\Identities,\Resources,\Profiles)$ into decisions over actions.

\section{Preliminaries}

The guiding requirement is monotonicity: adding governance context should never widen what an agent is allowed to do.

The policy algebra rests on six security principles. First, when additional governance context is introduced, the resulting execution posture should only become equally or more restrictive. Second, when execution crosses an agent, service, or workflow boundary, delegated authority should narrow rather than expand. Third, an action should be executable only when all independent policy families agree, including identity, agent configuration, service exposure, data policy, budget, artifact materialization, threat-layer policy, and human approval. Fourth, approval and evidence obligations should propagate through a call chain, because a later delegation must not erase obligations introduced earlier. Fifth, probabilistic risk estimates should guide ranking, review thresholds, and assurance investment, but they should not override hard policy feasibility. Sixth, budgeted unattended execution should have bounded loss: the runtime should remove paths that can spend the full budget without leaving a recoverable artifact. These principles build on established ideas in protection systems, information-flow control, and role-based access control \cite{saltzer-schroeder,denning-lattice,ferraiolo-rbac}, but agentic systems require them to be stated over dynamic reasoning-to-action transitions rather than over static users and resources alone.

\begin{figure}[!htbp]
\centering
\resizebox{\textwidth}{!}{%
\begin{tikzpicture}[
  font=\sffamily\footnotesize,
  >={Latex[length=2.0mm]},
  principle/.style={rounded corners=3pt, draw=bluegray!62, fill=bluegray!6, line width=0.45pt, minimum width=3.65cm, minimum height=0.78cm, align=left, inner sep=4pt},
  operator/.style={rounded corners=3pt, draw=deepblue!72, fill=deepblue!7, line width=0.55pt, minimum width=3.75cm, minimum height=0.78cm, align=center, inner sep=4pt},
  effect/.style={rounded corners=3pt, draw=sage!70, fill=sage!8, line width=0.45pt, minimum width=3.95cm, minimum height=0.78cm, align=left, inner sep=4pt},
  flow/.style={->, draw=black!45, line width=0.45pt},
  head/.style={font=\sffamily\bfseries\small, text=black!78, align=center}
]
\node[head] at (0,2.75) {Security principle};
\node[head] at (4.70,2.75) {Algebraic control};
\node[head] at (9.45,2.75) {Runtime effect};

\node[principle] (p1) at (0,1.95) {\textbf{Monotone context}\\[-0.1em]\scriptsize governance can only tighten};
\node[operator] (o1) at (4.70,1.95) {$p_{\mathrm{eff}}=\join_i p_i$};
\node[effect] (e1) at (9.45,1.95) {\textbf{Profile floor}\\[-0.1em]\scriptsize no lower layer weakens policy};

\node[principle] (p2) at (0,0.95) {\textbf{Non-widening delegation}\\[-0.1em]\scriptsize child authority cannot expand};
\node[operator] (o2) at (4.70,0.95) {$T_i\cap T_j,\;\Theta_i\cap\Theta_j,\;\min B$};
\node[effect] (e2) at (9.45,0.95) {\textbf{Narrowed authority}\\[-0.1em]\scriptsize tools, memory, and budget shrink};

\node[principle] (p3) at (0,-0.05) {\textbf{Independent agreement}\\[-0.1em]\scriptsize all policy families must pass};
\node[operator] (o3) at (4.70,-0.05) {$\rho\land\alpha\land\sigma\land\mu\land\delta\land\beta\land\eta$};
\node[effect] (e3) at (9.45,-0.05) {\textbf{Action gate}\\[-0.1em]\scriptsize one failed predicate blocks execution};

\node[principle] (p4) at (0,-1.05) {\textbf{Obligation propagation}\\[-0.1em]\scriptsize approval and evidence persist};
\node[operator] (o4) at (4.70,-1.05) {$h_i\lor h_j,\;E_i\cup E_j$};
\node[effect] (e4) at (9.45,-1.05) {\textbf{HITL and audit}\\[-0.1em]\scriptsize obligations survive call chains};

\node[principle] (p5) at (0,-2.05) {\textbf{Hard-policy dominance}\\[-0.1em]\scriptsize risk cannot authorize violations};
\node[operator] (o5) at (4.70,-2.05) {$\mathsf{Feasible}(a,s,c,\varpi)=1$};
\node[effect] (e5) at (9.45,-2.05) {\textbf{Feasible set}\\[-0.1em]\scriptsize risk ranks only allowed actions};

\node[principle] (p6) at (0,-3.05) {\textbf{Bounded loss}\\[-0.1em]\scriptsize no full spend without artifact};
\node[operator] (o6) at (4.70,-3.05) {$\mathsf{Materialize}(\xi_t,C_t,B_0)=1$};
\node[effect] (e6) at (9.45,-3.05) {\textbf{Recoverable output}\\[-0.1em]\scriptsize high-cost runs must materialize};

\foreach \i in {1,...,6} {
  \draw[flow] (p\i.east) -- (o\i.west);
  \draw[flow] (o\i.east) -- (e\i.west);
}
\end{tikzpicture}%
}
\caption{Policy-algebra principles and runtime effects. The algebra maps governance principles to monotone operators over profile, tool authority, memory scope, budget, artifact materialization, approval, and evidence obligations, yielding runtime effects that preserve trust before each reasoning-to-action transition.}
\label{fig:principles-algebra}
\end{figure}
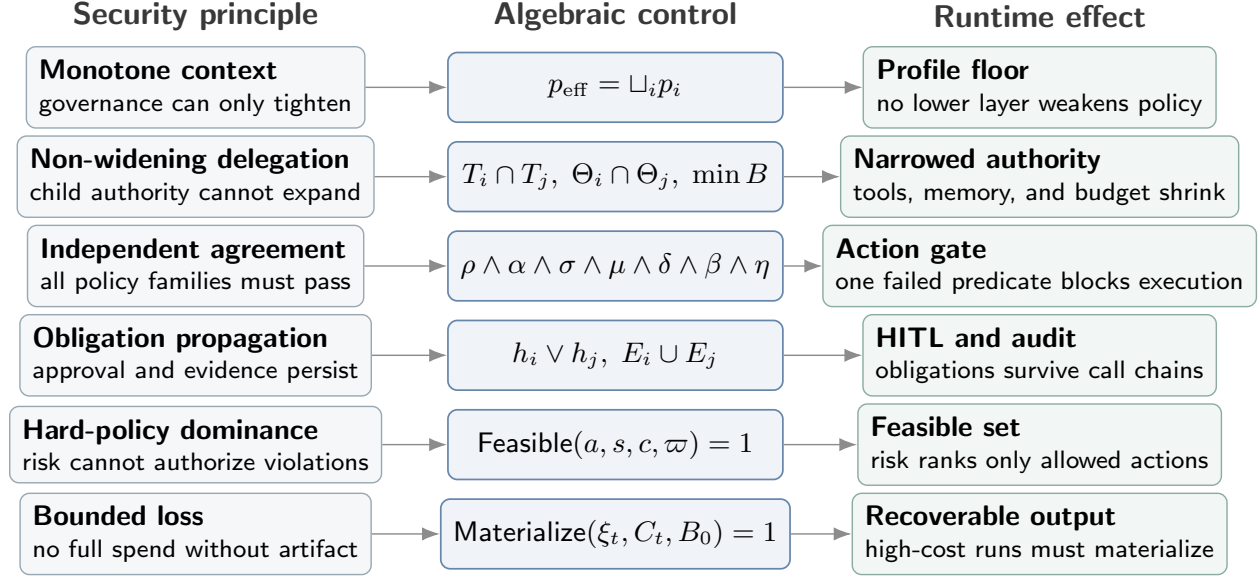

\subsection{The Policy Algebra}

The algebra is constructed so that composition narrows authority, accumulates obligations, and preserves evidence requirements.

Let $\mathcal{U}$ denote the universe of executable tools and services, $\mathcal{S}$ the universe of admissible memory and data scopes, $\mathcal{Q}$ the universe of artifact-materialization policies ordered by restrictiveness, and $\mathcal{E}$ the universe of audit evidence obligations. A runtime policy state is represented by the tuple
\[
    \varpi=(p,T,\Theta,B,Q,h,E),
\]
where $p\in\Profiles$ is the security profile, $T\subseteq\mathcal{U}$ is the executable tool/service authority, $\Theta\subseteq\mathcal{S}$ is the admissible memory and data scope, $B\in\mathbb{R}_{\ge0}$ is the remaining budget, $Q\in\mathcal{Q}$ is the artifact-materialization policy, $h\in\{0,1\}$ indicates whether human approval is required, and $E\subseteq\mathcal{E}$ is the set of evidence obligations that must be recorded. The algebraic order $\preceq_{\mathrm{alg}}$ is defined by
\begin{equation}
    \varpi_1\preceq_{\mathrm{alg}}\varpi_2
    \iff
    p_1\preceq p_2
    \land T_2\subseteq T_1
    \land \Theta_2\subseteq \Theta_1
    \land B_2\le B_1
    \land Q_1\preceq_{\mathcal{Q}} Q_2
    \land h_1\le h_2
    \land E_1\subseteq E_2 .
    \label{eq:algebra-order}
\end{equation}
where $\varpi_1=(p_1,T_1,\Theta_1,B_1,Q_1,h_1,E_1)$ and $\varpi_2=(p_2,T_2,\Theta_2,B_2,Q_2,h_2,E_2)$. Thus $\varpi_2$ is at least as restrictive and at least as auditable as $\varpi_1$ when it has a profile no weaker than $\varpi_1$, no additional tool authority, no broader memory scope, no larger budget, no weaker materialization policy, no weaker approval obligation, and no smaller evidence obligation.

\begin{definition}[Policy composition]
For two policy states $\varpi_1=(p_1,T_1,\Theta_1,B_1,Q_1,h_1,E_1)$ and $\varpi_2=(p_2,T_2,\Theta_2,B_2,Q_2,h_2,E_2)$, define their composition as
\begin{equation}
    \varpi_1\otimes\varpi_2
    =
    \left(
    p_1\join p_2,\;
    T_1\cap T_2,\;
    \Theta_1\cap \Theta_2,\;
    \min(B_1,B_2),\;
    Q_1\join_{\mathcal{Q}} Q_2,\;
    h_1\lor h_2,\;
    E_1\cup E_2
    \right).
    \label{eq:policy-composition}
\end{equation}
Here $\otimes$ denotes policy composition, $\join$ returns the stricter profile, $\cap$ narrows tool and memory authority, $\min(\cdot)$ selects the smaller budget, $\join_{\mathcal{Q}}$ returns the stricter materialization policy, $\lor$ accumulates approval requirements, and $\cup$ accumulates evidence obligations.
\end{definition}

\begin{lemma}[Algebraic closure and compositional laws]
If $\varpi_1$ and $\varpi_2$ are valid policy states, then $\varpi_1\otimes\varpi_2$ is a valid policy state. Moreover, $\otimes$ is commutative, associative, and idempotent.
\end{lemma}
\begin{proof}
Closure follows because $p_1\join p_2\in\Profiles$, $T_1\cap T_2\subseteq\mathcal{U}$, $\Theta_1\cap\Theta_2\subseteq\mathcal{S}$, $\min(B_1,B_2)\in\mathbb{R}_{\ge0}$, $Q_1\join_{\mathcal{Q}}Q_2\in\mathcal{Q}$, $h_1\lor h_2\in\{0,1\}$, and $E_1\cup E_2\subseteq\mathcal{E}$. Commutativity, associativity, and idempotence follow componentwise from the corresponding properties of least-upper-bound join, set intersection, minimum, Boolean disjunction, and set union.
\end{proof}

\begin{theorem}[Trust preservation under composition]
For any two valid policy states $\varpi_1$ and $\varpi_2$,
\[
    \varpi_1\preceq_{\mathrm{alg}}(\varpi_1\otimes\varpi_2)
    \quad\text{and}\quad
    \varpi_2\preceq_{\mathrm{alg}}(\varpi_1\otimes\varpi_2).
    \label{eq:trust-preservation}
\]
Consequently, policy composition cannot reduce profile assurance, add executable tools or services, broaden memory scope, increase budget, weaken artifact-materialization requirements, remove human approval, or remove audit obligations.
\end{theorem}
\begin{proof}
The least-upper-bound property gives $p_i\preceq p_1\join p_2$ and $Q_i\preceq_{\mathcal{Q}} Q_1\join_{\mathcal{Q}}Q_2$ for $i\in\{1,2\}$. Also $T_1\cap T_2\subseteq T_i$, $\Theta_1\cap\Theta_2\subseteq\Theta_i$, and $\min(B_1,B_2)\le B_i$. Boolean disjunction satisfies $h_i\le h_1\lor h_2$, and set union satisfies $E_i\subseteq E_1\cup E_2$. Substituting these componentwise relations into \eqref{eq:algebra-order} proves the result.
\end{proof}

\begin{theorem}[Least-restrictive sound composition]
Let $\varpi$ be any valid policy state satisfying $\varpi_1\preceq_{\mathrm{alg}}\varpi$ and $\varpi_2\preceq_{\mathrm{alg}}\varpi$. Then
\begin{equation}
    \varpi_1\otimes\varpi_2
    \preceq_{\mathrm{alg}}
    \varpi.
    \label{eq:least-restrictive-composition}
\end{equation}
Consequently, $\varpi_1\otimes\varpi_2$ is the least upper bound of the two governing states: it satisfies both inputs without imposing restrictions that are not required by at least one input component.
\end{theorem}
\begin{proof}
Because $\varpi$ is an upper bound, its profile and materialization components upper-bound both inputs; therefore $p_1\join p_2\preceq p$ and $Q_1\join_{\mathcal Q}Q_2\preceq_{\mathcal Q}Q$. Its tool and memory sets are subsets of both input sets, hence $T\subseteq T_1\cap T_2$ and $\Theta\subseteq\Theta_1\cap\Theta_2$. Similarly, $B\le B_1$ and $B\le B_2$ imply $B\le\min(B_1,B_2)$; $h_1\le h$ and $h_2\le h$ imply $h_1\lor h_2\le h$; and $E_1\subseteq E$ and $E_2\subseteq E$ imply $E_1\cup E_2\subseteq E$. These componentwise relations establish \eqref{eq:least-restrictive-composition} under \eqref{eq:algebra-order}.
\end{proof}

\begin{definition}[Action feasibility]
For an action $a$ in state $s$ under context $c$ and policy state $\varpi$, define
\begin{equation}
    \begin{aligned}
    \mathsf{Feasible}(a,s,c,\varpi)
    ={}& \mathsf{Auth}(a,c)
    \land \mathsf{Tool}(a,c,\varpi)
    \land \mathsf{Data}(a,c,\varpi)\\
    &\land \mathsf{Budget}(a,c,\varpi)
    \land \mathsf{Materialize}(a,c,\varpi)
    \land \mathsf{Approval}(a,c,\varpi)
    \land \mathsf{Evidence}(a,c,\varpi).
    \end{aligned}
    \label{eq:feasibility-formal}
\end{equation}
where $\mathsf{Auth}$ denotes identity and base authorization, $\mathsf{Tool}$ denotes tool or service admissibility, $\mathsf{Data}$ denotes data-class admissibility, $\mathsf{Budget}$ denotes budget feasibility, $\mathsf{Materialize}$ denotes cost-aware artifact materialization feasibility, $\mathsf{Approval}$ denotes satisfaction of human-approval requirements, and $\mathsf{Evidence}$ denotes availability of the required audit record.
\end{definition}

\begin{definition}[Admissible action set]
For state $s_t$, context $c_t$, and policy state $\varpi_t$, define
\begin{equation}
    \mathcal{A}_{\mathrm{adm}}(s_t,c_t,\varpi_t)
    =
    \{a\in\Actions:
    \mathsf{Feasible}(a,s_t,c_t,\varpi_t)=1
    \land
    \mathsf{Risk}(a,s_t,c_t)\le \rho_{\max}(p_t)\}.
    \label{eq:admissible-action-set}
\end{equation}
The admissible set is the action space remaining after hard policy predicates and the profile-dependent risk ceiling have been applied.
\end{definition}

The set $\mathcal{A}_{\mathrm{adm}}$ is the \emph{reliability envelope}. It separates the responsibility of the policy runtime from that of the reasoner: the runtime determines which actions are admissible, while the reasoner remains free to optimize task utility among those actions.

\begin{theorem}[Gate soundness and policy-relative maximal permissiveness]
For a fixed state, context, policy state, and risk ceiling, the admissible-action definition has two properties:
\begin{align}
    a\in\mathcal A_{\mathrm{adm}}
    &\Rightarrow
    \mathsf{Feasible}(a,s,c,\varpi)=1,
    \label{eq:gate-soundness}\\
    \mathsf{Feasible}(a,s,c,\varpi)=1
    \land \mathsf{Risk}(a,s,c)\le\rho_{\max}(p)
    &\Rightarrow
    a\in\mathcal A_{\mathrm{adm}}.
    \label{eq:gate-maximality}
\end{align}
Thus the gate excludes policy-violating actions but does not exclude an action that satisfies the complete declared feasibility conjunction and risk ceiling. This is maximal permissiveness relative to the encoded policy, not a claim that the encoded policy is complete.
\end{theorem}
\begin{proof}
Both implications follow directly from the set definition in \eqref{eq:admissible-action-set}. The qualification is necessary because incorrect or missing context predicates remain possible as described in \eqref{eq:conditional-correctness}.
\end{proof}

\begin{lemma}[Hard-policy dominance]
If $a\notin \mathcal{A}_{\mathrm{adm}}(s_t,c_t,\varpi_t)$ because $\mathsf{Feasible}(a,s_t,c_t,\varpi_t)=0$, then $x_{t,a}=0$ in any feasible solution of the controlled action-selection problem \eqref{eq:objective}.
\end{lemma}
\begin{proof}
The optimization constraints in \eqref{eq:objective} require $\mathsf{Feasible}(a,s_t,c_t,\varpi_t)=1$ for every selected action with $x_{t,a}=1$. If $\mathsf{Feasible}(a,s_t,c_t,\varpi_t)=0$, selecting $a$ would violate this constraint. Therefore any feasible solution must set $x_{t,a}=0$. Utility, cost, and residual-risk terms affect only the ranking of actions that remain in $\mathcal{A}_{\mathrm{adm}}(s_t,c_t,\varpi_t)$.
\end{proof}

\subsection{Profile Cascade}

The effective profile is the most restrictive profile induced by all applicable governance scopes.

Enterprise governance appears at several scopes. For request context $c$, the effective profile is
\begin{equation}
    p_{\mathrm{eff}}(c)=
    p_{\mathrm{platform}}
    \join p_{\mathrm{team}}
    \join p_{\mathrm{user}}
    \join p_{\mathrm{agent}}
    \join p_{\mathrm{publication}}.
    \label{eq:profile-cascade}
\end{equation}
where $p_{\mathrm{platform}}$, $p_{\mathrm{team}}$, $p_{\mathrm{user}}$, $p_{\mathrm{agent}}$, and $p_{\mathrm{publication}}$ are the profiles induced by platform policy, team policy, user policy, agent configuration, and publication approval, respectively. This equation means that any layer may tighten the execution posture, but no lower layer may loosen inherited constraints.

\begin{algorithm}[t]
\caption{Profile cascade resolution}
\label{alg:profile}
\begin{algorithmic}[1]
\Require Ordered profiles from platform, team, user, agent, and publication scopes
\Ensure Effective profile $p_{\mathrm{eff}}$
\State $p_{\mathrm{eff}}\gets p_{\mathrm{platform}}$
\For{$p$ in $(p_{\mathrm{team}},p_{\mathrm{user}},p_{\mathrm{agent}},p_{\mathrm{publication}})$}
    \State $p_{\mathrm{eff}}\gets p_{\mathrm{eff}}\join p$
\EndFor
\State \Return $p_{\mathrm{eff}}$
\end{algorithmic}
\end{algorithm}

\begin{lemma}[Monotone profile resolution]
If $\join$ is the least upper bound on $(\Profiles,\preceq)$, then adding a new governing scope to \eqref{eq:profile-cascade} cannot reduce the assurance level of the effective profile.
\end{lemma}
\begin{proof}
For any $p,q\in\Profiles$, $p\preceq p\join q$ by definition of least upper bound. Repeated application over the cascade gives $p_{\mathrm{eff}}\preceq p_{\mathrm{eff}}\join q$ for any additional governing profile $q$.
\end{proof}

\subsection{Policy-Intersected Tool Authorization}

A tool call is treated as a governed transaction, not as a permission attached only to an agent.

A call $(\tau,\theta)$ is permitted only when independent policies agree:
\begin{align}
\mathsf{Permit}(\tau,\theta,c)=
&~\rho(u,\tau,R_t)
\land \alpha(g,\tau)
\land \sigma(\tau,p_{\mathrm{publication}},E_p) \notag\\
&\land \mu(\tau,\theta,L_{\mathrm{MAESTRO}})
\land \delta(\tau,\theta,C_p)
\land \beta(\tau,B_t)
\land \kappa(\tau,\theta,Q_t)
\land \eta(\tau,\theta,H_t).
\label{eq:permit}
\end{align}
Here $\tau$ is the tool or service being called, $\theta$ is its parameter vector, $u$ is the initiating identity, and $g$ is the agent or service context. The predicate $\rho$ is RBAC permission, $\alpha$ the agent/workflow allowlist, $\sigma$ service exposure policy, $\mu$ layer or threat-model policy, $\delta$ data policy, $\beta$ budget admissibility, $\kappa$ artifact-materialization admissibility, and $\eta$ approval satisfaction.

\begin{algorithm}[t]
\caption{Tool authorization under intersected policies}
\label{alg:permission}
\begin{algorithmic}[1]
\Require Tool $\tau$, parameters $\theta$, context $c$, policy set $\Policies$
\Ensure $\Decision{Allow}$, $\Decision{Deny}$, or $\Decision{RequireApproval}$
\If{$\rho(u,\tau,R_t)=0$} \State \Return \Decision{Deny} \EndIf
\If{$\alpha(g,\tau)=0$} \State \Return \Decision{Deny} \EndIf
\If{$\sigma(\tau,p_{\mathrm{publication}},E_p)=0$} \State \Return \Decision{Deny} \EndIf
\If{$\mu(\tau,\theta,L_{\mathrm{MAESTRO}})=0$} \State \Return \Decision{Deny} \EndIf
\If{$\delta(\tau,\theta,C_p)=0$} \State \Return \Decision{Deny} \EndIf
\If{$\beta(\tau,B_t)=0$} \State \Return \Decision{Deny} \EndIf
\If{$H_p(\tau,\theta)=1$ and $\mathsf{Approved}(\tau,\theta,c)=0$}
    \State \Return \Decision{RequireApproval}
\EndIf
\State \Return \Decision{Allow}
\end{algorithmic}
\end{algorithm}

\subsection{RAG and Memory Access}

Retrieved context and memory are state variables in the execution path. Their use must therefore be admissible under the same policy state that governs tools.

Let $d\in\Data$ be a retrieved document or context fragment, and let $m\in\Memory$ be a proposed memory write. Admissibility is defined as a conjunction over classification, ownership, scope, purpose, and retention:
\begin{equation}
\begin{aligned}
\mathsf{Read}_{p}(u,g,d,c)
={}&
\indicator\{
\kappa(d)\preceq C_p
\land u\in \mathsf{ACL}(d)\\
&\land \mathsf{scope}(d)\cap\Theta_p(g)\ne\emptyset
\land \mathsf{purpose}(c)\in \mathsf{Purpose}_p
\}.
\end{aligned}
\label{eq:read-admissible}
\end{equation}
\begin{equation}
\begin{aligned}
\mathsf{Write}_{p}(u,g,m,c)
={}&
\indicator\{
\kappa(m)\preceq C_p
\land \mathsf{owner}(m)=u\\
&\land \mathsf{scope}(m)\subseteq\Theta_p(g)
\land \mathsf{ttl}(m)\le \tau_p
\land \mathsf{sanitize}_p(m)=m
\}.
\end{aligned}
\label{eq:write-admissible}
\end{equation}
where $\kappa(\cdot)$ gives the data classification, $C_p$ is the maximum admissible class under profile $p$, $\mathsf{ACL}(d)$ is the access-control list for document $d$, $\mathsf{scope}(\cdot)$ gives the business or memory scope, $\Theta_p(g)$ is the scope allowed to agent $g$ under profile $p$, $\mathsf{Purpose}_p$ is the set of allowed purposes, $\mathsf{owner}(m)$ is the memory owner, $\mathsf{ttl}(m)$ is the memory time-to-live, $\tau_p$ is the retention bound, and $\mathsf{sanitize}_p$ is the profile-specific sanitization map.
A context construction step is admissible only when
\[
    \prod_{d\in \mathsf{Context}(a_t)} \mathsf{Read}_{p_t}(u,g,d,c_t)=1,
\]
and a memory update is admissible only when
\[
    \prod_{m\in \mathsf{MemWrite}(a_t)} \mathsf{Write}_{p_t}(u,g,m,c_t)=1.
\]
Here $\mathsf{Context}(a_t)$ is the set of retrieved fragments used by action $a_t$, and $\mathsf{MemWrite}(a_t)$ is the set of memory objects proposed for persistence by action $a_t$.
These formulations bind memory and RAG access to the same profile algebra used for tools.

\subsection{Trust-Preserving Delegation}

Delegation is a change of execution context. The delegated context must inherit restrictions rather than obtain fresh authority.

When agent or service $i$ invokes $j$, the child context is derived by
\begin{align}
    p_{i\rightarrow j} &= p_i\join p_j, \label{eq:trust-profile}\\
    T_{i\rightarrow j} &= T_i\cap T_j, \label{eq:trust-tools}\\
    B_{i\rightarrow j} &= \min(B_i^{\mathrm{remaining}},B_j), \label{eq:trust-budget}\\
    H_{i\rightarrow j}(a) &= H_i(a)\lor H_j(a), \label{eq:trust-hitl}\\
    \Theta_{i\rightarrow j} &= \Theta_i \cap \Theta_j. \label{eq:trust-memory}
\end{align}
where $p_i$ and $p_j$ are the caller and callee profiles, $T_i$ and $T_j$ are their tool sets, $B_i^{\mathrm{remaining}}$ is the caller's remaining budget, $B_j$ is the callee's budget cap, $H_i$ and $H_j$ are approval predicates, and $\Theta_i$ and $\Theta_j$ are memory-scope relations. The profile join prevents profile laundering; the set intersections and minimum budget prevent authority expansion through delegation.

Equation~\eqref{eq:trust-budget} is a per-edge cap. In a branching or concurrent call graph, per-edge narrowing alone is insufficient because several children could each be assigned the same remaining parent budget. Let $J_i(t)$ be the set of child calls opened by agent $i$ at time $t$, let $b_{i\rightarrow j}$ be the budget atomically reserved for child $j$, and let $b_i^{\mathrm{local}}$ be the amount retained for the parent. A budget-conserving executor additionally enforces
\begin{equation}
    b_i^{\mathrm{local}}
    +\sum_{j\in J_i(t)}b_{i\rightarrow j}
    \le B_i^{\mathrm{remaining}},
    \qquad
    b_{i\rightarrow j}\le B_{i\rightarrow j}.
    \label{eq:delegation-budget-conservation}
\end{equation}
This treats budget as a conserved execution resource: delegation may divide or reduce it, but parallelism cannot multiply it.

\begin{lemma}[Delegation cannot lower the effective profile]
For any call $i\rightarrow j$, the delegated profile satisfies $p_i\preceq p_{i\rightarrow j}$ and $p_j\preceq p_{i\rightarrow j}$.
\end{lemma}
\begin{proof}
The result follows from \eqref{eq:trust-profile} and the least-upper-bound property of $\join$.
\end{proof}

\begin{lemma}[Delegation cannot amplify allocated budget]
If every branching step satisfies \eqref{eq:delegation-budget-conservation}, then the sum of budgets allocated to the parent and its immediate children does not exceed the parent's remaining budget at that step.
\end{lemma}
\begin{proof}
The claim is exactly the conservation inequality in \eqref{eq:delegation-budget-conservation}. Reapplying it at every child recursively prevents any subtree from creating budget not allocated by its parent. Atomic reservation is required so concurrent children cannot reserve the same units.
\end{proof}

\subsection{Multi-Agent Extension}

The multi-agent case is obtained by applying the same monotone rule along a call path.

Let a multi-agent call graph be $G=(V,E)$, where $V$ is the set of agents or services and $E$ contains directed calls. For a path $\pi=(v_0,v_1,\ldots,v_k)$, the path profile and tool authority are
\[
    p(\pi)=\join_{\ell=0}^{k}p_{v_\ell},
    \qquad
    T(\pi)=\bigcap_{\ell=0}^{k}T_{v_\ell}.
\]
where $p_{v_\ell}$ is the profile at node $v_\ell$ and $T_{v_\ell}$ is the tool authority at that node. The path profile $p(\pi)$ is the strictest profile appearing on the path, and $T(\pi)$ is the common tool authority available to all nodes on the path.
The path is admissible if
\[
    \mathsf{Depth}(\pi)\le D_{\max},\quad
    T(\pi)\ne\emptyset\ \text{for required actions},\quad
    \bigvee_{\ell=0}^{k}H_{v_\ell}(a)\Rightarrow \mathsf{Approved}(a,\pi).
\]
Here $\mathsf{Depth}(\pi)$ is the call depth, $D_{\max}$ is the maximum permitted depth, $H_{v_\ell}(a)$ is the approval requirement for action $a$ at node $v_\ell$, and $\mathsf{Approved}(a,\pi)$ is the approval evidence attached to the path.

\begin{algorithm}[t]
\caption{Evaluate a multi-agent call chain}
\label{alg:multi-agent}
\begin{algorithmic}[1]
\Require Call path $\pi=(v_0,\ldots,v_k)$, action $a$, maximum depth $D_{\max}$
\Ensure Derived path context or denial
\If{$k>D_{\max}$} \State \Return \Decision{Deny} \EndIf
\State $p_\pi\gets p_{v_0}$; $T_\pi\gets T_{v_0}$; $B_\pi\gets B_{v_0}^{\mathrm{remaining}}$; $h_\pi\gets H_{v_0}(a)$
\For{$\ell=1$ to $k$}
    \State $p_\pi\gets p_\pi\join p_{v_\ell}$
    \State $T_\pi\gets T_\pi\cap T_{v_\ell}$
    \State $B_\pi\gets \min(B_\pi,B_{v_\ell})$
    \State $h_\pi\gets h_\pi\lor H_{v_\ell}(a)$
\EndFor
\If{$T_\pi$ lacks a required tool for $a$} \State \Return \Decision{Deny} \EndIf
\If{$h_\pi=1$ and $\mathsf{Approved}(a,\pi)=0$} \State \Return \Decision{RequireApproval} \EndIf
\State \Return $(p_\pi,T_\pi,B_\pi,h_\pi)$
\end{algorithmic}
\end{algorithm}

\subsection{Residual Risk Model}

Risk is used to rank feasible actions, not to relax hard feasibility constraints.

The objective in the next subsection uses a risk term, but the runtime does not rely on an opaque scalar. Risk is decomposed by threat class. Let $\Threats=\{z_1,\ldots,z_m\}$ be the set of threat categories, for example the T1--T15 categories from the Agentic Security Hub. For an action $a$ in state $s$ and context $c$, the risk decomposition uses four quantities, where $q_z(a,s,c)\in[0,1]$ is the conditional likelihood that threat $z$ is relevant or materializes for the action; $I_z(a,s,c)\in\mathbb{R}_{\ge 0}^{r}$ is the impact vector, for example confidentiality, integrity, availability, privacy, compliance, and business impact; $\omega\in\mathbb{R}_{\ge0}^{r}$ is the enterprise impact-weight vector; and $\psi_z(a,s,c)\in[0,1]$ is the residual exposure after active controls are applied.
The residual risk of an action is
\begin{equation}
    \mathsf{Risk}(a,s,c)=
    \sum_{z\in\Threats}
    q_z(a,s,c)\,
    \big(\omega^\top I_z(a,s,c)\big)\,
    \psi_z(a,s,c).
    \label{eq:residual-risk}
\end{equation}
where $z$ indexes a threat class in $\Threats$, $q_z(a,s,c)$ is the likelihood term, $\omega^\top I_z(a,s,c)$ is the scalarized impact, and $\psi_z(a,s,c)$ is the residual exposure after controls are applied.
The residual exposure term is profile-dependent because profiles activate controls:
\begin{equation}
    \psi_z(a,s,c)=
    \prod_{m\in \mathcal{C}(a,c,p_{\mathrm{eff}})}
    \left(1-\epsilon_{z,m}\right),
    \label{eq:residual-exposure}
\end{equation}
where $\mathcal{C}(a,c,p_{\mathrm{eff}})$ is the set of controls active for action $a$ under the effective profile, and $\epsilon_{z,m}\in[0,1]$ is the estimated effectiveness of control $m$ against threat $z$. For example, sandboxing reduces residual exposure for unexpected code execution, HITL reduces exposure for high-impact external actions, and memory scoping reduces exposure for memory poisoning or data leakage.

This formulation separates four quantities that are often collapsed incorrectly: threat likelihood, impact, active controls, and residual exposure. It also makes estimation practical. Initial values can be obtained from threat-modeler mappings, action type, tool criticality, data classification, architecture pattern, call depth, external exposure, and profile level. A simple estimator may use
\begin{equation}
    \widehat{q}_z(a,s,c)=
    \sigma\!\left(\theta_z^\top \phi(a,s,c)\right),
    \label{eq:likelihood-estimator}
\end{equation}
where $\sigma(\cdot)$ is a link function, $\theta_z$ is the parameter vector for threat class $z$, and $\phi(a,s,c)$ contains features such as tool type, data class, caller role, service exposure, memory scope, and whether the action has external side effects. When historical telemetry exists, the likelihood can be updated with incident and red-team evidence, for example with a beta-binomial update:
\begin{equation}
    \widehat{q}_{z,t+1}=
    \frac{\alpha_z+n_{z,t}^{\mathrm{event}}}
         {\alpha_z+\beta_z+n_{z,t}^{\mathrm{trial}}}.
    \label{eq:bayes-risk-update}
\end{equation}
where $\alpha_z$ and $\beta_z$ are prior parameters for threat class $z$, $n_{z,t}^{\mathrm{event}}$ is the number of observed events by time $t$, and $n_{z,t}^{\mathrm{trial}}$ is the corresponding number of trials or opportunities for that threat class.
The exact estimator is implementation-dependent; the algebra requires only that risk be decomposed into estimable components. Hard security requirements remain in $\mathsf{Feasible}(\cdot)$ and $\mathsf{Permit}(\cdot)$, so an imprecise risk estimate cannot authorize an action that violates policy. Risk is used to rank feasible alternatives, set review thresholds, and prioritize controls.

\subsection{Valid Artifacts and Recoverability}

Artifact existence alone is too weak for bounded-loss execution: an empty file or unusable checkpoint should not satisfy the materialization obligation. Let $o$ be a candidate artifact and let $Q$ be the applicable materialization policy. Define
\begin{equation}
\begin{aligned}
\mathsf{ValidArtifact}(o,Q)={}&
\mathsf{Durable}(o)
\land\mathsf{Nontrivial}(o,Q)
\land\mathsf{SchemaValid}(o,Q)\\
&\land\mathsf{Provenanced}(o,Q)
\land\mathsf{Resumable}(o,Q).
\end{aligned}
\label{eq:valid-artifact}
\end{equation}
The exact tests are profile- and workflow-specific. A ticket may require a persisted identifier and status; a report may require a nonempty validated schema and source references; a checkpoint may require sufficient state to resume without repeating the full consumed cost. In the algorithms below, $\mathsf{Artifact}(\xi_t)$ is shorthand for the existence of an artifact that satisfies the operational validity checks configured by $Q_t$.

A stronger materialization deployment can reserve the cost of producing such an artifact. Let $R_Q(s_t)$ be a conservative upper bound on the cost of at least one available materialization action from state $s_t$. A non-materializing action $a_t$ is budget-admissible only if
\begin{equation}
    C_t+\mathsf{cost}(a_t)+R_Q(s_{t+1})\le B_0.
    \label{eq:materialization-reserve}
\end{equation}
Threshold-based redirection in Algorithm~\ref{alg:materialization} is a practical approximation of this reserve rule when a workflow-specific cost bound is unavailable.

\begin{theorem}[Reserved capacity for recoverable execution]
Assume cost estimates upper-bound realized action cost, reserve updates are atomic, and at least one artifact-producing action costing no more than $R_Q(s_t)$ remains available. If every non-materializing action satisfies \eqref{eq:materialization-reserve}, then open-ended execution cannot consume the capacity reserved to attempt a valid artifact.
\end{theorem}
\begin{proof}
After any admitted non-materializing action, \eqref{eq:materialization-reserve} leaves at least $R_Q(s_{t+1})$ of the initial budget unconsumed. By assumption, an artifact-producing action can be attempted within that reserve. The theorem preserves capacity to attempt materialization; artifact success still depends on the materializer and storage boundary and is therefore recorded as an explicit operational assumption rather than inferred from budget alone.
\end{proof}

\subsection{Controlled Action Selection}

Autonomy is modeled as optimization over the admissible set in \eqref{eq:admissible-action-set}. Utility matters only after the policy constraints are satisfied.

The runtime receives candidate actions from a reasoner and chooses among actions in $\mathcal{A}_{\mathrm{adm}}(s_t,c_t,\varpi_t)$. Let $x_{t,a}\in\{0,1\}$ denote whether action $a$ is selected at step $t$, and let $\varpi_t$ be the policy state induced by the effective context at that step. The execution objective is
\begin{align}
\max_{x_{t,a}}\quad
&\sum_{t=0}^{T}\sum_{a\in\Actions}
\Bigl(
u(a,s_t)-\lambda \mathsf{Risk}(a,s_t,c_t) \notag\\
&\qquad\qquad
-\gamma\mathsf{cost}(a)
\Bigr)x_{t,a}
\label{eq:objective}\\
\text{s.t.}\quad
&\sum_{a\in\Actions}x_{t,a}=1,\quad \forall t,\notag\\
&\mathsf{Feasible}(a,s_t,c_t,\varpi_t)=1,\quad \forall (t,a):x_{t,a}=1,\notag\\
&\mathsf{Risk}(a,s_t,c_t)\le \rho_{\max}(p_t),\quad \forall (t,a):x_{t,a}=1,\notag\\
&\sum_{t,a}\mathsf{cost}(a)x_{t,a}\le B_{p_0},\notag\\
&\mathsf{Materialize}(\xi_t,C_t,B_{p_0},Q_t)=1,\quad \forall t,\notag\\
&F(\xi)=0,\notag\\
&x_{t,a}\in\{0,1\}.\notag
\end{align}
where $u(a,s_t)$ is the task utility of action $a$ in state $s_t$, $\lambda$ and $\gamma$ are non-negative weights on risk and cost, $\mathsf{cost}(a)$ is the resource cost of action $a$, $\rho_{\max}(p_t)$ is the maximum admissible risk under profile $p_t$, $B_{p_0}$ is the initial profile budget, $C_t=\sum_{\ell=0}^{t}\sum_{a\in\Actions}\mathsf{cost}(a)x_{\ell,a}$ is cumulative consumed cost, and $Q_t$ is the applicable artifact-materialization policy. The feasibility and risk constraints are equivalent to requiring any selected action to lie in $\mathcal{A}_{\mathrm{adm}}(s_t,c_t,\varpi_t)$. The objective states the semantics of governed autonomy. The agent may pursue utility, but only inside the feasible region defined by the policy algebra. The threshold $\rho_{\max}(p_t)$ is stricter for higher-assurance profiles and can be set conservatively when risk estimates are immature.

The budget constraint prevents overspend, but it does not by itself guarantee a useful payoff. A run can respect the cap and still end with no durable artifact. The materialization constraint prevents this zero-deliverable exhaustion state. Let $\mathsf{Artifact}(\xi_t)=1$ denote that the trace up to step $t$ has produced a durable and recoverable artifact satisfying the operational checks associated with \eqref{eq:valid-artifact}, such as a draft, checkpoint, file, ticket, record, or resumable intermediate state. The predicate $\mathsf{Materialize}(\xi_t,C_t,B_{p_0},Q_t)$ is satisfied when a valid artifact already exists, when cost consumption remains below the materialization threshold, or when the selected action is artifact-producing. Thus $\mathsf{Budget}(\cdot)$ prevents overspend, while $\mathsf{Materialize}(\cdot)$ prevents spending the available budget without leaving a recoverable output.

This mechanism is a stop-loss rather than a workflow prescription. The runtime does not require the model to write before it has enough information. It allows research, retrieval, and planning while cost exposure remains acceptable. Once cost exposure becomes material and no artifact exists, the admissible action set is redirected toward producing a recoverable draft or checkpoint before further open-ended reasoning.

\begin{algorithm}[t]
\caption{Governed execution loop}
\label{alg:harness}
\begin{algorithmic}[1]
\Require Request, caller identity, agent configuration, policy set $\Policies$
\Ensure Response and audit trace
\State Authenticate caller and load RBAC context
\State $p_{\mathrm{eff}}\gets$ Algorithm~\ref{alg:profile}
\State Initialize trace, call-chain metadata, memory scope, and budget
\While{not terminated}
    \State Build context using \eqref{eq:read-admissible}
    \State Generate candidate action from an allowed model $M_{p_{\mathrm{eff}}}$
    \If{candidate is a tool call}
        \State Apply Algorithm~\ref{alg:permission}
    \ElsIf{candidate is an agent or service call}
        \State Apply Algorithm~\ref{alg:multi-agent}
    \ElsIf{candidate is a memory write}
        \State Enforce \eqref{eq:write-admissible}
    \EndIf
    \State Apply Algorithm~\ref{alg:materialization} using current cost, budget, artifact state, and proposed action
    \State Execute only if the selected action is feasible; otherwise deny, request approval, or ask for an alternative
    \State Record audit evidence for identity, profile, policy decision, data class, approval, cost, and result
\EndWhile
\State Classify and redact output according to $p_{\mathrm{eff}}$
\State Persist trace and return response
\end{algorithmic}
\end{algorithm}

\begin{algorithm}[t]
\caption{Cost-aware artifact materialization}
\label{alg:materialization}
\begin{algorithmic}[1]
\Require Trace $\xi_t$, cumulative cost $C_t$, budget $B_0$, proposed action $a_t$, thresholds $\alpha,\beta$ with $0<\alpha<\beta<1$
\Ensure Allow, redirect, restrict, or stop decision
\State $r_t\gets C_t/B_0$
\State $\mathit{hasArtifact}\gets \mathsf{Artifact}(\xi_t)$
\If{$C_t\ge B_0 \land \mathit{hasArtifact}$}
    \State \Return \Decision{StopSuccess}
\ElsIf{$C_t\ge B_0 \land \neg\mathit{hasArtifact}$}
    \State \Return \Decision{StopFailure}
\ElsIf{$\mathit{hasArtifact}$}
    \State \Return \Decision{Allow}
\ElsIf{$r_t<\alpha$}
    \State \Return \Decision{Allow}
\ElsIf{$\alpha\le r_t<\beta$}
    \State \Return \Decision{RedirectToMaterialize}
\ElsIf{$r_t\ge\beta \land \mathsf{ProducesArtifact}(a_t)$}
    \State \Return \Decision{Allow}
\Else
    \State \Return \Decision{RestrictToArtifactWrite}
\EndIf
\end{algorithmic}
\end{algorithm}

\subsection{Consequences of the Algebra}

\begin{lemma}[Human approval cannot be bypassed by delegation]
If an action $a$ requires approval at any node on a call path $\pi$, then Algorithm~\ref{alg:multi-agent} returns \Decision{RequireApproval} unless approval evidence is present.
\end{lemma}
\begin{proof}
The path approval flag $h_\pi$ is the disjunction of all node-level approval predicates. If any node requires approval, $h_\pi=1$. The algorithm checks approval before returning an executable context.
\end{proof}

\begin{lemma}[Tool authority narrows along a call path]
For any path $\pi=(v_0,\ldots,v_k)$, $T(\pi)\subseteq T_{v_\ell}$ for every $\ell$.
\end{lemma}
\begin{proof}
The path tool set is the intersection of all tool sets on the path.
\end{proof}

\section{Enterprise Runtime Architecture}

The architecture is a realization of the algebra, not a separate source of security. Its purpose is to place each predicate on an enforceable runtime boundary.

The algebra is implemented as a platform architecture rather than as a standalone checker. The architecture separates three concerns: a framework layer containing reusable blueprints and patterns, a runtime execution layer containing the system of record, and an application layer containing user interfaces, workflow builders, MCP clients, REST clients, and tenant applications. This separation keeps reusable agent design patterns distinct from production enforcement.

Core platform services include an agent registry, profile and RBAC management, orchestration, tool and connector registries, knowledge and RAG services, memory management, publication and service governance, audit, monitoring, cost governance, artifact-state tracking, and deployment controls. Cross-cutting controls include RBAC/IAM, policy enforcement, audit and compliance, risk and threat management, budget and quota enforcement, artifact materialization, privacy and data protection, human approval, and continuous verification.

The cost-governance and artifact-state services are evaluated jointly. Budget enforcement records actual cost at each reasoning step, tool call, model call, and workflow node. Artifact-state tracking records whether the run has produced a durable and recoverable output, such as a draft document, checkpoint, file, database record, ticket, or resumable intermediate state. A run that has consumed substantial budget but has no artifact is not treated as merely expensive; it is treated as an unsafe payoff state. The platform therefore constrains the next admissible actions toward materialization before allowing additional open-ended research or planning.

This distinction matters for unattended workflows. In a terminal-based coding assistant, the human developer can observe the run and interrupt it. In an asynchronous enterprise workflow, the system must be the governor. The runtime, not the model, enforces the stop-loss.

\begingroup
\setlength{\intextsep}{0.35\baselineskip}
\begin{figure}[H]
\centering
\resizebox{\textwidth}{!}{%
\begin{tikzpicture}[
  font=\sffamily\footnotesize,
  >={Latex[length=2.2mm]},
  input/.style={rounded corners=3pt, draw=bluegray!70, fill=bluegray!7, line width=0.45pt, minimum width=3.15cm, minimum height=1.05cm, align=center, inner sep=4pt},
  policy/.style={rounded corners=3pt, draw=deepblue!70, fill=deepblue!8, line width=0.55pt, minimum width=3.25cm, minimum height=0.82cm, align=center, inner sep=4pt},
  gate/.style={rounded corners=3pt, draw=slate!75, fill=slate!7, line width=0.55pt, minimum width=3.45cm, minimum height=2.00cm, align=center, inner sep=5pt},
  outcome/.style args={#1}{rounded corners=3pt, draw=#1!75, fill=#1!9, line width=0.55pt, minimum width=2.70cm, minimum height=0.70cm, align=center, inner sep=3pt},
  flow/.style={->, draw=black!55, line width=0.55pt},
  weakflow/.style={->, draw=black!40, line width=0.45pt, dashed},
  title/.style={font=\sffamily\bfseries\small, text=black!78},
  note/.style={font=\sffamily\scriptsize, text=black!62, align=center}
]

\node[input] (artifacts) at (0,1.65) {\textbf{Security guidance}\\[-0.1em]\scriptsize AISVS, NIST, MAESTRO\\threats and patterns};
\node[input] (context) at (0,0) {\textbf{Enterprise context}\\[-0.1em]\scriptsize identity, RBAC, team\\user, agent, publication};
\node[input] (metadata) at (0,-1.65) {\textbf{Runtime metadata}\\[-0.1em]\scriptsize tools, data, memory\\services, budget, artifact, HITL};

\node[draw=black!12, fill=palegray, rounded corners=5pt, minimum width=3.75cm, minimum height=4.95cm] (algbox) at (4.30,0) {};
\node[title, anchor=north] at ($(algbox.north)+(0,-0.18)$) {Policy composition};
\node[policy] (cascade) at (4.30,1.18) {Profile cascade\\[-0.1em]\scriptsize $p_{\mathrm{eff}}=\join_i p_i$};
\node[policy] (narrow) at (4.30,0.05) {Authority narrowing\\[-0.1em]\scriptsize $T_i\cap T_j$, $\Theta_i\cap\Theta_j$, $\min B$};
\node[policy] (obligations) at (4.30,-1.08) {Obligation inheritance\\[-0.1em]\scriptsize HITL and audit evidence};
\node[note] at ($(algbox.south)+(0,0.38)$) {monotone by construction};

\node[draw=deepblue!75, fill=deepblue!7, rounded corners=4pt, line width=0.6pt, minimum width=3.35cm, minimum height=1.08cm, align=center] (state) at (8.25,0.95) {\textbf{Effective policy state}\\[-0.1em]\scriptsize $\varpi_t=(p,T,\Theta,B,Q,h,E)$};
\node[gate] (gate) at (8.25,-0.75) {\textbf{Action gate}\\[-0.1em]\scriptsize Id $\land$ Role $\land$ Prof $\land$ Data\\[-0.1em]\scriptsize Mem $\land$ Tool $\land$ Bud $\land$ Art\\[-0.1em]\scriptsize HITL $\land$ Aud};

\node[draw=black!12, fill=palegray, rounded corners=5pt, minimum width=3.30cm, minimum height=4.30cm] (outbox) at (12.25,0) {};
\node[title, anchor=north] at ($(outbox.north)+(0,-0.18)$) {Decision};
\node[outcome=rose] (deny) at (12.25,1.05) {Deny and log};
\node[outcome=amber] (approve) at (12.25,0) {Require approval};
\node[outcome=sage] (execute) at (12.25,-1.05) {Execute constrained};

\node[draw=black!20, fill=white, rounded corners=3pt, line width=0.45pt, minimum width=3.70cm, minimum height=0.74cm, align=center] (trace) at (8.25,-2.95) {\textbf{Trace and repair}\\[-0.1em]\scriptsize lineage, cost, artifact, failed predicate};

\draw[flow] (artifacts.east) -- (algbox.west |- artifacts.east);
\draw[flow] (context.east) -- (algbox.west);
\draw[flow] (metadata.east) -- (algbox.west |- metadata.east);
\draw[flow] (algbox.east) -- (state.west);
\draw[flow] (state.south) -- (gate.north);
\draw[flow] (gate.east) -- (deny.west);
\draw[flow] (gate.east) -- (approve.west);
\draw[flow] (gate.east) -- (execute.west);
\draw[flow] (deny.east) -- ++(0.55,0) |- (trace.east);
\draw[flow] (approve.east) -- ++(0.55,0) |- (trace.east);
\draw[flow] (execute.east) -- ++(0.55,0) |- (trace.east);
\draw[weakflow] (trace.west) -| (algbox.south);
\end{tikzpicture}%
}
\caption{Runtime enforcement flow for the policy algebra. Standards, threat findings, enterprise context, and runtime metadata are compiled into a monotone policy state. Each proposed action is admitted only when the identity, role, profile, data, memory, tool, budget, artifact-materialization, approval, and audit predicates jointly hold; denials, approvals, redirects, and executions all emit trace evidence for repair and regression testing.}
\label{fig:policy-algebra-runtime}
\end{figure}
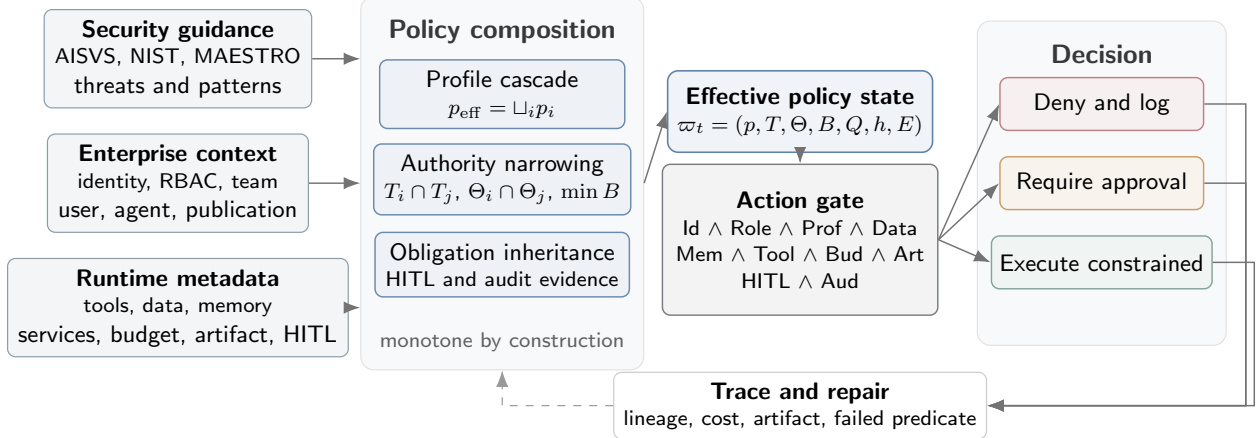

\begin{table}[H]
\centering
\caption{Security profile semantics used by the runtime.}
\label{tab:profiles}
\small
\begin{tabularx}{\textwidth}{@{}p{0.06\textwidth}p{0.14\textwidth}YYYY@{}}
\toprule
No. & Profile & Models/tools & Memory/data & Approval & Audit/budget \\
\midrule
1 & \Low & Safe defaults; built-in tools & Short-lived memory; no sensitive RAG & None by default & Basic logs \\
2 & \Medium & Approved models/tools; limited connectors & Retention, redaction, classified RAG & Destructive actions & Tool parameters, costs, per-run budget \\
3 & \High & Whitelisted tools/connectors & Limited memory; vector decay; strict classes & External side effects and writes & Step traces, rate limits, anomaly detection \\
4 & \VeryHigh & Deterministic tools for regulated actions & Sanitized knowledge only; strict retention & All regulated/external actions & Immutable evidence, fixed budgets, alerts \\
\bottomrule
\end{tabularx}
\end{table}
\endgroup

\subsection{Implementation Sketches}

The following snippets show how a runtime may compile the algebra into concrete enforcement artifacts. The placeholders marked \texttt{<>} indicate deployment-specific values.

\begin{lstlisting}[language=Python,caption={Profile configuration sketch.}]
profiles:
  high:
    allowed_models: [<approved-models>]
    allowed_tools: [<tool-ids>]
    data_classes: [internal, confidential]
    memory:
      retention_days: <N>
      cross_session: false
      vector_index: profile_scoped
    hitl:
      required_for: [external_write, destructive_action]
    audit:
      depth: step_level
      evidence: [identity, profile, policy, tool, params_hash, approval, cost]
\end{lstlisting}

\begin{lstlisting}[language=Python,caption={Permission-gate sketch.}]
def permit(tool, params, ctx):
    checks = [
        role_permitted(ctx.user, tool),
        agent_allows(ctx.agent, tool),
        service_exposed(ctx.publication, tool),
        layer_policy_ok(tool, params, ctx.maestro_layer),
        data_policy_ok(params, ctx.profile),
        budget_ok(tool, ctx.budget),
        approval_ok(tool, params, ctx.hitl_state),
    ]
    if all(checks):
        return Allow()
    if approval_required(tool, params, ctx.profile):
        return RequireApproval(reason=<policy_reason>)
    return Deny(reason=<failed_predicate>)
\end{lstlisting}

\begin{lstlisting}[language=Python,caption={Audit-event sketch.}]
{
  "run_id": "<run>",
  "actor": "<user-or-service-account>",
  "agent_id": "<agent>",
  "effective_profile": "High",
  "action": "<tool-or-service-call>",
  "policy_decision": "RequireApproval",
  "data_class": "<classification>",
  "approval_id": "<approval-record>",
  "budget_before": "<amount>",
  "budget_after": "<amount>",
  "trace_parent": "<parent-step>",
  "result_hash": "<hash>"
}
\end{lstlisting}

\begin{table}[!htbp]
\centering
\caption{Reference implementation status.}
\label{tab:implementation}
\small
\begin{tabularx}{\textwidth}{@{}p{0.06\textwidth}p{0.30\textwidth}X@{}}
\toprule
No. & Component & Implementation status \\
\midrule
1 & Security profiles and cascade & Implemented in configuration and UI; formal semantics defined in \eqref{eq:profile-cascade} \\
2 & Policy-intersected tool authorization & Implemented in permission gate and executor paths \\
3 & Trust boundary propagation & Implemented for profile and depth; formal verification remains future work \\
4 & Human approval inheritance & Implemented for core paths; asynchronous workflow queues under development \\
5 & RAG and memory governance & Classification and access checks exist; advanced cross-session retention remains partial \\
6 & Publication profile freezing & Implemented using publication snapshots and review triggers \\
7 & Audit and observability & Baseline metadata exists; immutable storage and richer anomaly detection remain future work \\
8 & Infrastructure controls & Implemented through Kubernetes/EKS hardening, mTLS, OPA, secrets, and encrypted stores \\
\bottomrule
\end{tabularx}
\end{table}

\section{Threat-to-Control Mapping}

Threat findings become useful only when they are converted into predicates, profile floors, and evidence obligations.

Let $\mathcal{F}_{\mathrm{threat}}$ be the set of threat-model findings. For each finding $f\in\mathcal{F}_{\mathrm{threat}}$, define the threat-policy compiler
\begin{equation}
    \Gamma_{\mathrm{threat}}(f)=(m_f,p_f,g_f,e_f),
    \label{eq:threat-policy-compiler}
\end{equation}
where $m_f$ is the selected mitigation family, $p_f\in\Profiles$ is the minimum profile required for the affected action, $g_f$ is the executable runtime gate or predicate, and $e_f\subseteq\mathcal{E}$ is the evidence set that must be recorded when the gate is evaluated. The runtime evaluates $g_f$ before the affected action executes and records $e_f$ for allow, deny, and approval outcomes.
For example, an unexpected remote-code-execution finding for an agent using KC6 operational-environment components yields a sandboxing mitigation, command allowlist, network isolation, a High or Very High profile floor, denial or approval before execution, and audit evidence for actor, command, sandbox, approval state, and profile.

\begin{table}[!htbp]
\centering
\caption{Threat-to-control mapping with explicit runtime controls and audit evidence.}
\label{tab:threats}
\scriptsize
\begin{tabularx}{\textwidth}{@{}p{0.04\textwidth}p{0.14\textwidth}p{0.12\textwidth}p{0.13\textwidth}p{0.13\textwidth}X p{0.15\textwidth}@{}}
\toprule
No. & Threat & Layers & AISVS categories & Components & Runtime controls & Audit evidence \\
\midrule
1 & Memory poisoning & Data, framework, observability & C8, C12 & Memory, RAG, tools & Save/publish/run validation; retrieval audit; classification; retention limits & Source ID, classifier result, retrieval hash, action result \\
2 & Tool misuse & Tools, runtime & C5, C9, C12 & Tool integration, runtime & Intersected permission; sandbox; schema validation; HITL & Tool, params hash, policy predicates, approver, cost \\
3 & Privilege compromise & Framework, ecosystem & C5, C9, C10 & Orchestration, runtime & RBAC; cascade; monotone trust; depth limits & Caller, callee, profile join, denied/escalated reason \\
4 & Resource overload & Runtime, observability & C12 & Runtime & Token/cost budgets; rate limits; loop controls & Budget before/after, rate decision, loop counter \\
5 & Overwhelming HITL & Framework, ecosystem & C13 & Orchestration & Risk-adaptive approval; batching; deterministic tools & Approval queue state, decision latency, approver \\
6 & Unexpected RCE & Tools, runtime & C4, C5, C9 & Tool integration, runtime & Sandbox; command allowlist; network blocking & Command hash, sandbox ID, network policy, result \\
\bottomrule
\end{tabularx}
\end{table}

\section{Evaluation Methodology}
\label{sec:evaluation}

The evaluation is organized around conformance rather than benchmark ranking. A runtime is evaluated by asking whether the required predicates fire and whether the resulting trace contains sufficient evidence.

The paper reports a preliminary design evaluation and defines a reproducible runtime evaluation protocol for the proposed policy-algebra standard. It does not treat any public benchmark suite as the source of validity. Instead, validity is defined by conformance to the proposed control taxonomy, policy invariants, executable predicates, and evidence requirements. The runtime results in Section~\ref{sec:results} evaluate the proposed policy-algebra runtime against an RBAC-and-allowlist comparator using the same workload, metrics, and scenario categories. As the conformance workload, adversarial cases, and production-like service-publication traces expand, the same tables and metrics can be updated with the latest measured values. The evaluation therefore separates design coverage, runtime enforcement, adversarial scenarios, and policy repair.

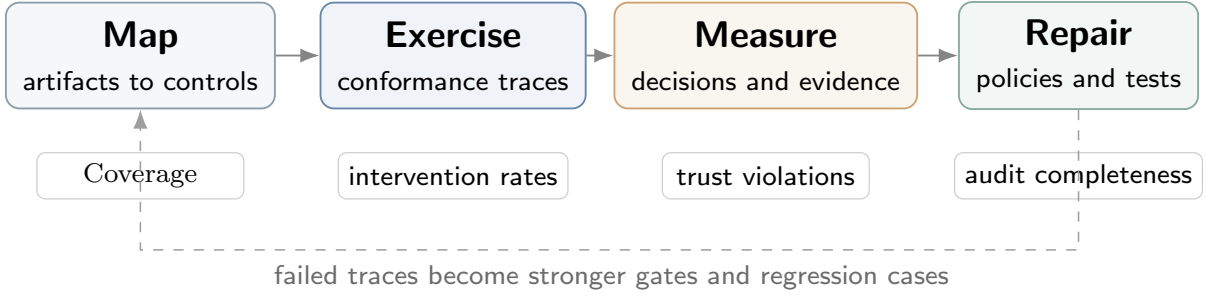
\begin{figure}[!htbp]
\centering
\resizebox{0.96\textwidth}{!}{%
\begin{tikzpicture}[
  font=\sffamily,
  >={Latex[length=2.0mm]},
  stage/.style={rounded corners=4pt, line width=0.55pt, minimum width=2.45cm, minimum height=1.10cm, align=center, inner sep=5pt},
  metric/.style={rounded corners=2.5pt, draw=black!20, fill=white, line width=0.35pt, minimum width=2.15cm, minimum height=0.48cm, align=center, inner sep=3pt, font=\sffamily\scriptsize},
  flow/.style={->, draw=black!48, line width=0.55pt},
  loop/.style={->, draw=black!38, line width=0.45pt, dashed}
]
\node[stage, draw=bluegray!70, fill=bluegray!6] (map) at (0,0) {\textbf{Map}\\[-0.1em]\scriptsize artifacts to controls};
\node[stage, draw=deepblue!72, fill=deepblue!7] (exercise) at (3.25,0) {\textbf{Exercise}\\[-0.1em]\scriptsize conformance traces};
\node[stage, draw=amber!72, fill=amber!8] (measure) at (6.50,0) {\textbf{Measure}\\[-0.1em]\scriptsize decisions and evidence};
\node[stage, draw=sage!72, fill=sage!8] (repair) at (9.75,0) {\textbf{Repair}\\[-0.1em]\scriptsize policies and tests};

\draw[flow] (map.east) -- (exercise.west);
\draw[flow] (exercise.east) -- (measure.west);
\draw[flow] (measure.east) -- (repair.west);

\node[metric] at (0,-1.25) {$\mathrm{Coverage}$};
\node[metric] at (3.25,-1.25) {intervention rates};
\node[metric] at (6.50,-1.25) {trust violations};
\node[metric] at (9.75,-1.25) {audit completeness};

\coordinate (feedback) at (4.90,-2.02);
\draw[loop] (repair.south) |- (feedback) -| (map.south);
\node[font=\sffamily\scriptsize, text=black!58, align=center] at (4.90,-2.30) {failed traces become stronger gates and regression cases};
\end{tikzpicture}%
}
\caption{Conformance evaluation and repair protocol for governed agents. The evaluation maps security artifacts to executable controls, exercises conformance traces, measures decisions and evidence, and converts failed traces into stronger gates and regression cases. Event-level details such as trace identifier, decision, failed predicate, profile, tool or service, approval state, budget, and latency are recorded in the trace rather than crowded into the diagram.}
\label{fig:evaluation-repair-protocol}
\end{figure}

The evaluation addresses seven research questions.
\begin{enumerate}[leftmargin=*]
    \item \textbf{RQ1: Coverage.} Does each standard control, threat class, component class, and architecture pattern map to at least one executable runtime control?
    \item \textbf{RQ2: Compilation.} Does each mapped control compile into a runtime predicate, gate, audit field, profile floor, or regression test?
    \item \textbf{RQ3: Tool control.} Does the permission gate block unsafe tool execution paths in the test protocol?
    \item \textbf{RQ4: Delegation.} Does the trust algebra prevent profile downgrades and profile laundering in multi-agent calls?
    \item \textbf{RQ5: Evidence.} Does the audit layer record enough data to reconstruct the actor, profile, tool or service, source, approval state, and policy decision?
    \item \textbf{RQ6: Bounded-loss execution.} Does cost-aware artifact materialization reduce executions in which budget is exhausted without a durable output?
    \item \textbf{RQ7: Repair.} Can failed evaluation cases be converted into policy updates and regression tests?
\end{enumerate}

\begin{hypothesis}[Coverage]
The proposed control taxonomy maps each standard control, threat class, component class, and architecture pattern considered in the evaluation to at least one executable runtime control.
\end{hypothesis}

\begin{hypothesis}[Compilation]
Each mapped security artifact can be compiled into at least one runtime predicate, gate, audit field, profile floor, or regression test.
\end{hypothesis}

\begin{hypothesis}[Tool-control effectiveness]
Policy-intersected authorization will detect unsafe tool and service invocations that are missed by single-layer authorization mechanisms such as RBAC-only, service-exposure-only, or tool-allowlist-only enforcement.
\end{hypothesis}

\begin{hypothesis}[Delegation control]
A runtime that enforces profile joins, tool intersections, memory-scope intersections, budget narrowing, and HITL inheritance will produce fewer trust-widening violations than a runtime using independent RBAC checks and tool allowlists alone.
\end{hypothesis}

\begin{hypothesis}[Audit reconstruction]
Executions that emit evidence for identity, profile, policy decision, data class, approval state, budget state, artifact state, call-chain lineage, and result hash will support more reproducible governance review than executions with run-level logging only.
\end{hypothesis}

\begin{hypothesis}[Artifact materialization]
A runtime that enforces cost-aware materialization thresholds will produce fewer zero-artifact budget-exhaustion failures than a runtime that enforces budget limits only.
\end{hypothesis}

\begin{hypothesis}[Policy repair]
Failed evaluation cases can be converted into policy updates and regression tests.
\end{hypothesis}

\begin{table}[!htbp]
\centering
\caption{Mapping from research questions and hypotheses to evaluation evidence.}
\label{tab:hypothesis-evidence}
\small
\begin{tabularx}{\textwidth}{@{}p{0.09\textwidth}p{0.18\textwidth}X p{0.25\textwidth}@{}}
\toprule
Research question & Hypothesis & Metric, equation, or evidence & Evaluation location \\
\midrule
RQ1 & H1: Coverage & $\mathrm{Coverage}(F,C)$ in Eq.~\eqref{eq:coverage}; design-coverage matrix & Section~\ref{sec:prelim-design-eval} \\
RQ2 & H2: Compilation & $\mathrm{mapped}(f_i)$ in Eq.~\eqref{eq:mapped}; runtime predicate, gate, audit-field, profile-floor, or regression-test mapping & Section~\ref{sec:prelim-design-eval} \\
RQ3 & H3: Tool control & Unsafe intervention rate and false intervention rate in Eq.~\eqref{eq:blockrate} and Eq.~\eqref{eq:falseblock}; governed tool-call protocol & Section~\ref{sec:runtimeprotocol}; Section~\ref{sec:results} \\
RQ4 & H4: Delegation & Trust violation rate in Eq.~\eqref{eq:trustviolation}; profile-monotonicity violations & Section~\ref{sec:runtimeprotocol}; Section~\ref{sec:results} \\
RQ5 & H5: Evidence & Audit completeness in Eq.~\eqref{eq:auditcompleteness}; trace reconstruction fields & Section~\ref{sec:results} \\
RQ6 & H6: Artifact materialization & Zero-artifact exhaustion rate; materialization redirects and artifact-state evidence & Section~\ref{sec:runtimeprotocol}; Section~\ref{sec:results} \\
RQ7 & H7: Repair & Policy update and regression-update rules in Eq.~\eqref{eq:policyrepair} and Eq.~\eqref{eq:regressionupdate} & Section~\ref{sec:policyrepair} \\
\bottomrule
\end{tabularx}
\tablesource{Authors' formulation}
\end{table}

Let $F=\{f_1,\ldots,f_n\}$ be the set of security artifacts to be checked. In this paper, $F$ contains AISVS categories, NIST AI RMF functions, MAESTRO layers, T1--T15 threat classes, KC1--KC6 component classes, and architecture patterns. Let $C=\{c_1,\ldots,c_m\}$ be the set of executable runtime controls. A finding $f_i$ is mapped if at least one runtime control enforces, monitors, or records it:
\begin{equation}
    \mathrm{mapped}(f_i)=
    \begin{cases}
    1, & \exists c_j\in C:\, \mathrm{enforces}(c_j,f_i)\lor \mathrm{monitors}(c_j,f_i)\lor \mathrm{audits}(c_j,f_i),\\
    0, & \text{otherwise.}
    \end{cases}
    \label{eq:mapped}
\end{equation}
where $\mathrm{enforces}(c_j,f_i)$ means that control $c_j$ blocks or permits behavior associated with finding $f_i$, $\mathrm{monitors}(c_j,f_i)$ means that the control observes the relevant behavior, and $\mathrm{audits}(c_j,f_i)$ means that the control records evidence for it.
The design coverage score is
\begin{equation}
    \mathrm{Coverage}(F,C)=\frac{1}{|F|}\sum_{i=1}^{|F|}\mathrm{mapped}(f_i).
    \label{eq:coverage}
\end{equation}
where $|F|$ is the number of artifact classes being checked. The score is one only when every artifact class has at least one mapped enforcement, monitoring, or audit path.
Coverage is not proof of security. It states whether the runtime has at least one control path for each item. Runtime testing is needed to check whether the control fires under adversarial execution.

For runtime evaluation, let $A_{\mathrm{unsafe}}$ be the set of actions that violate at least one policy predicate and let $A_{\mathrm{safe}}$ be the set of actions that should be allowed. An \emph{intervention} is any decision that prevents immediate unrestricted execution: denial, required approval, sandboxing or restriction, or redirection to artifact materialization. Equivalently, an intervention occurs when $\delta(a)\ne\Decision{Allow}$. The unsafe intervention rate and false intervention rate are
\begin{equation}
    \mathrm{UnsafeInterventionRate}=\frac{|\{a\in A_{\mathrm{unsafe}}:\delta(a)\ne\Decision{Allow}\}|}{|A_{\mathrm{unsafe}}|},
    \label{eq:blockrate}
\end{equation}
\begin{equation}
    \mathrm{FalseInterventionRate}=\frac{|\{a\in A_{\mathrm{safe}}:\delta(a)\ne\Decision{Allow}\}|}{|A_{\mathrm{safe}}|},
    \label{eq:falseblock}
\end{equation}
where $\delta(a)$ is the runtime decision for action $a$, $A_{\mathrm{unsafe}}$ is the set of policy-violating actions, and $A_{\mathrm{safe}}$ is the set of actions expected to be allowed. Intervention is intentionally broader than denial because approval and materialization decisions are central behaviors of the proposed runtime; scenario-level results should therefore be read as immediate execution prevented or constrained, not necessarily as permanent rejection. For a delegation chain $g_0\rightarrow g_1\rightarrow\cdots\rightarrow g_k$ with effective profiles $p_0,\ldots,p_k$, the trust violation rate is
\begin{equation}
    \mathrm{TrustViolationRate}=\frac{|\{i\in\{0,\ldots,k-1\}:p_{i+1}\prec p_i\}|}{k}.
    \label{eq:trustviolation}
\end{equation}
where $g_i$ is the $i$th agent or service on the chain, $p_i$ is its effective profile, and $p_{i+1}\prec p_i$ denotes a strict profile downgrade. A correct implementation of the monotonic trust rule should have $\mathrm{TrustViolationRate}=0$. For audit evidence, let $R(e)$ be the required audit fields for event $e$ and let $L(e)$ be the logged fields. Audit completeness is
\begin{equation}
    \mathrm{AuditCompleteness}=\frac{|\{e\in E:R(e)\subseteq L(e)\}|}{|E|}.
    \label{eq:auditcompleteness}
\end{equation}
where $E$ is the event set, $R(e)$ is the required evidence set for event $e$, and $L(e)$ is the set of fields actually logged for that event.

For bounded-loss execution, let $\mathcal{R}$ be the set of completed or stopped runs, let $\mathsf{Exhausted}(r)$ indicate that run $r$ consumed its available budget, and let $\mathsf{Artifact}(r)$ indicate that run $r$ produced a durable recoverable output. The zero-artifact exhaustion rate is
\begin{equation}
    \mathrm{ZeroArtifactExhaustionRate}
    =
    \frac{
    |\{r\in\mathcal{R}:\mathsf{Exhausted}(r)\land \neg\mathsf{Artifact}(r)\}|
    }{
    |\{r\in\mathcal{R}:\mathsf{Exhausted}(r)\}|
    }.
    \label{eq:zeroartifact}
\end{equation}
A correct materialization implementation should drive this rate toward zero by redirecting or restricting actions before budget exhaustion.

The evaluation produces three outputs: a coverage matrix, a set of weakly covered or uncovered cases, and a set of policy repair actions. These outputs feed back into the policy set. A failed case becomes a stricter profile floor, a narrower tool allowlist, a new HITL rule, a memory access restriction, a service-publication review, a materialization threshold update, a new audit requirement, or a new regression test.

\section{Preliminary Design Evaluation}
\label{sec:prelim-design-eval}

The design evaluation asks whether the algebra has a control surface for each class of security artifact.

The preliminary design evaluation checks whether the proposed runtime has an executable control path for each artifact class. It is separate from the runtime measurements reported in Section~\ref{sec:results}. The aim here is narrower: verify whether the policy algebra has enough control surfaces to support runtime tests.

\begin{longtable}{@{}>{\raggedright\arraybackslash}p{0.04\textwidth}
>{\raggedright\arraybackslash}p{0.22\textwidth}
>{\raggedright\arraybackslash}p{0.32\textwidth}
>{\raggedright\arraybackslash}p{0.30\textwidth}@{}}
\caption{Preliminary design coverage of security artifacts by runtime controls.}\label{tab:designcoverage}\\
\toprule
No. & Artifact family & Coverage criterion & Runtime control path \\
\midrule
\endfirsthead
\toprule
No. & Artifact family & Coverage criterion & Runtime control path \\
\midrule
\endhead
1 & AISVS categories & Each category maps to an enforcement or audit surface & Profile rules, input validation, output filtering, memory policy, tool gate, audit schema \\
2 & NIST AI RMF functions & Each function maps to a lifecycle or runtime mechanism & Governance cascade, asset mapping, monitoring, policy repair \\
3 & MAESTRO layers & Each layer maps to one or more runtime controls & Model gateway, RAG gate, tool gate, sandbox, observability, trust chain \\
4 & T1--T15 threat classes & Each threat maps to a control, evidence field, and test case & Threat-to-control matrix, profile floor, HITL, monitoring, trace field \\
5 & KC1--KC6 components & Each component maps to scoped runtime constraints & Model, orchestration, memory, tool, and operational policies \\
6 & Agent architecture patterns & Each pattern maps to profile floors and delegation constraints & Sequential, hierarchical, swarm, reactive, and RAG agent policies \\
\bottomrule
\end{longtable}
\tablesource{Authors' design evaluation}

The next check evaluates whether the central invariants of the algebra are represented in the policy algebra and in the implementation design.

\begingroup
\setlength{\LTpre}{0.8em}
\setlength{\LTpost}{0pt}
\begin{longtable}{@{}>{\raggedright\arraybackslash}p{0.04\textwidth}
>{\raggedright\arraybackslash}p{0.21\textwidth}
>{\raggedright\arraybackslash}p{0.34\textwidth}
>{\raggedright\arraybackslash}p{0.29\textwidth}@{}}
\caption{Preliminary evaluation of policy-algebra invariants.}\label{tab:invariantcoverage}\\
\toprule
No. & Invariant & Formal condition & Runtime mechanism \\
\midrule
\endfirsthead
\toprule
No. & Invariant & Formal condition & Runtime mechanism \\
\midrule
\endhead
1 & Profile monotonicity & $p_i\preceq p_{i\rightarrow j}$ for delegated execution & Profile cascade and profile join \\
2 & Tool authorization & $\mathsf{Permit}(\tau,\theta,c)=1$ before execution & Intersected permission gate \\
3 & HITL inheritance & $H_{i\rightarrow j}(a)=H_i(a)\lor H_j(a)$ & Approval propagation over the call chain \\
4 & Budget narrowing & $B_{i\rightarrow j}\leq B_i$ & Budget attribution and sub-budget allocation \\
5 & Artifact materialization & $\mathsf{Materialize}(\xi_t,C_t,B_{p_0},Q_t)=1$ & Cost-governance and artifact-state gate \\
6 & Publication freeze & $p_{\mathrm{publication}}$ participates in $p_{\mathrm{eff}}$ until review & Published service profile snapshot \\
7 & RAG access control & $\mathsf{Read}_{p}(u,g,d,c)=1$ for retrieved context & Save-time, publish-time, and run-time validation \\
8 & Audit completeness & $R(e)\subseteq L(e)$ for each event & Audit schema and trace logger \\
\bottomrule
\end{longtable}
\tablesource{Authors' design evaluation}
\endgroup

This preliminary design evaluation identifies three remaining gaps. First, a mapped control does not prove attack resistance. A predicate may exist but still fail to fire under some execution path. Second, the policy algebra assumes correct metadata for tools, data sources, memory items, and services. A misclassified tool or data object can weaken enforcement. Third, semantic influence through memory and retrieval can persist even when exact text is filtered. These gaps motivate the runtime evaluation protocol in Section~\ref{sec:runtimeprotocol}.

\section{Runtime Evaluation Protocol}
\label{sec:runtimeprotocol}

The runtime protocol turns the algebra into testable action events.

The runtime protocol defines executable conformance tests for the runtime. It is intended for execution against the authors' conformance workload, internal red-team cases, production-like service-publication workflows, and optional external datasets used only for comparative context. Each test specifies an actor, agent, profile, action, expected decision, and expected audit record.

\subsection{Trace Evaluation}

\begin{lstlisting}[language=Python,caption={Trace-level evaluation sketch.}]
def evaluate_trace(trace):
    metrics = {
        "unsafe_attempts": 0,
        "intervened_unsafe": 0,
        "safe_attempts": 0,
        "intervened_safe": 0,
        "profile_monotonicity_violations": 0,
        "budget_violations": 0,
        "zero_artifact_exhaustions": 0,
        "hitl_violations": 0,
        "audit_complete": 0,
        "events": len(trace.events),
    }
    for event in trace.events:
        unsafe = event.ground_truth == "unsafe"
        blocked = event.decision == "deny"
        if unsafe:
            metrics["unsafe_attempts"] += 1
            if blocked:
                metrics["intervened_unsafe"] += 1
        else:
            metrics["safe_attempts"] += 1
            if blocked:
                metrics["intervened_safe"] += 1
        if event.type == "delegation" and event.callee_profile < event.caller_profile:
            metrics["profile_monotonicity_violations"] += 1
        if event.cost_after > event.budget:
            metrics["budget_violations"] += 1
        if event.budget_exhausted and not event.artifact_exists:
            metrics["zero_artifact_exhaustions"] += 1
        if event.hitl_required and not event.hitl_approved:
            metrics["hitl_violations"] += 1
        if event.required_audit_fields <= event.logged_fields:
            metrics["audit_complete"] += 1
    return metrics
\end{lstlisting}

\subsection{Tool Misuse Test}

\begin{lstlisting}[language=Python,caption={Blocked tool-call test sketch.}]
def test_blocked_tool_call(harness):
    user = User(role="viewer", profile="Medium")
    agent = Agent(profile="Medium", allowed_tools={"search"})
    action = ToolCall(name="delete_record", params={"id": "123"})

    decision = harness.evaluate(user=user, agent=agent, action=action)

    assert decision.status == "denied"
    assert decision.failed_factor in {"role", "agent_allowlist", "profile"}
    assert harness.audit.contains(user.id, agent.id, action.name, decision.status)
\end{lstlisting}

\subsection{Profile Laundering Test}

\begin{lstlisting}[language=Python,caption={Profile-laundering test sketch.}]
def test_profile_laundering_prevention(harness):
    parent = Agent(id="planner", profile="High", allowed_tools={"search", "write"})
    child = Agent(id="worker", profile="Low", allowed_tools={"search", "write", "email"})
    requested = ToolCall(name="email", params={"to": "external@example.com"})

    ctx = harness.resolve_delegation(parent=parent, child=child, action=requested)

    assert ctx.effective_profile == "High"
    assert "email" not in ctx.allowed_tools or ctx.hitl_required is True
    assert harness.audit.contains_chain(parent.id, child.id, ctx.effective_profile)
\end{lstlisting}

\subsection{RAG Access-Control Test}

\begin{lstlisting}[language=Python,caption={RAG access-control test sketch.}]
def test_rag_access_control(harness):
    user = User(role="finance_viewer", profile="Medium")
    agent = Agent(profile="Medium", scope={"finance"})
    document = Document(
        id="doc-hr-01",
        label="restricted",
        scope={"hr"},
        provenance="sharepoint",
        quality=0.95,
    )

    result = harness.retrieve(user=user, agent=agent, query="salary data", corpus=[document])

    assert document.id not in result.document_ids
    assert harness.audit.contains_policy_event("rag_access_denied")
\end{lstlisting}

\subsection{Publication-Profile Test}

\begin{lstlisting}[language=Python,caption={Publication-profile freeze test sketch.}]
def test_publication_profile_freeze(harness):
    service = PublishedService(
        id="supplier-risk-api",
        approved_profile="High",
        requested_profile="Medium",
        status="published",
    )

    decision = harness.invoke_service(service=service)

    assert decision.status == "review_required"
    assert decision.reason == "profile_change_after_publication"
    assert harness.audit.contains_service_event(service.id, decision.status)
\end{lstlisting}

\subsection{Artifact Materialization Test}

\begin{lstlisting}[language=Python,caption={Cost-aware artifact materialization test sketch.}]
def test_materialization_before_budget_exhaustion(harness):
    run = harness.start_workflow(budget=10.0)
    run.consume(cost=7.0, artifact_exists=False)

    decision = harness.evaluate_next_action(
        run=run,
        proposed_action=ResearchStep(query="continue research")
    )

    assert decision.status in {"redirect", "restricted"}
    assert decision.required_action == "materialize_artifact"
    assert harness.audit.contains(run.id, "materialization_required")
\end{lstlisting}

\begingroup
\setlength{\LTpre}{0.8em}
\setlength{\LTpost}{0pt}
\small
\renewcommand{\arraystretch}{1.06}
\begin{longtable}{@{}>{\raggedright\arraybackslash}p{0.05\textwidth}
>{\raggedright\arraybackslash}p{0.23\textwidth}
>{\raggedright\arraybackslash}p{0.24\textwidth}
>{\raggedright\arraybackslash}p{0.22\textwidth}
>{\raggedright\arraybackslash}p{0.16\textwidth}@{}}
\caption{Runtime evaluation protocol and expected outcome.}\label{tab:evaluationprotocol}\\
\toprule
No. & Scenario & Targeted rule & Expected decision & Evidence field \\
\midrule
\endfirsthead
\toprule
No. & Scenario & Targeted rule & Expected decision & Evidence field \\
\midrule
\endhead
1 & Prompt injection calls blocked tool & Eq.~\eqref{eq:permit} & Deny & Failed predicate \\
2 & Medium agent invokes Very High service & Eq.~\eqref{eq:profile-cascade}; Eq.~\eqref{eq:trust-profile} & Deny or run at stricter profile & Effective profile \\
3 & Sub-agent widens tools & Algorithm~\ref{alg:multi-agent} & Deny & Delegated tool set \\
4 & External app calls unpublished service & $\sigma$ in Eq.~\eqref{eq:permit} & Deny & Service status \\
5 & Sensitive data enters RAG context & Eq.~\eqref{eq:read-admissible} & Deny retrieval & Data class and clearance \\
6 & Memory poisoning attempt & Eq.~\eqref{eq:write-admissible} & Deny write & Provenance and redaction \\
7 & External write without HITL & $\eta$ in Eq.~\eqref{eq:permit}; Algorithm~\ref{alg:permission} & Review or deny & HITL state \\
8 & Code execution with network call & $\mu$ in Eq.~\eqref{eq:permit} & Deny or sandbox & Command and sandbox ID \\
9 & Budget exhaustion & $\beta$ in Eq.~\eqref{eq:permit} & Deny next costed action & Remaining budget \\
10 & Budget mostly consumed with no artifact & Algorithm~\ref{alg:materialization}; Eq.~\eqref{eq:zeroartifact} & Redirect or restrict to materialization & Artifact state and cost ratio \\
11 & Missing audit field & Eq.~\eqref{eq:reliable-predicate} & Mark trace unreliable & Audit completeness \\
\bottomrule
\end{longtable}
\tablesource{Authors' experimental design}
\endgroup

\section{Results and Discussion}
\label{sec:results}

The empirical question is whether policy composition changes the admissible action set in the intended direction while preserving useful capability: more unsafe actions should be intercepted, trust violations should decrease, evidence should become more complete, and task completion should remain operationally meaningful.

This section reports the runtime evaluation results for the proposed policy-algebra runtime and an RBAC-and-allowlist comparator. The workload covers conformance tasks, adversarial cases, and production-like service-publication traces. The results are reported at three levels: workload composition, aggregate enforcement metrics, and scenario-level unsafe-action intervention.

Figure~\ref{fig:runtime-workload-composition} provides a graphical view of the workload underlying Table~\ref{tab:runtime-workload}. The near-balanced safe and unsafe event counts in the prompt/tool family contrast with the larger safe-event share in the remaining families; trace counts are shown separately because a trace may contain several action events.

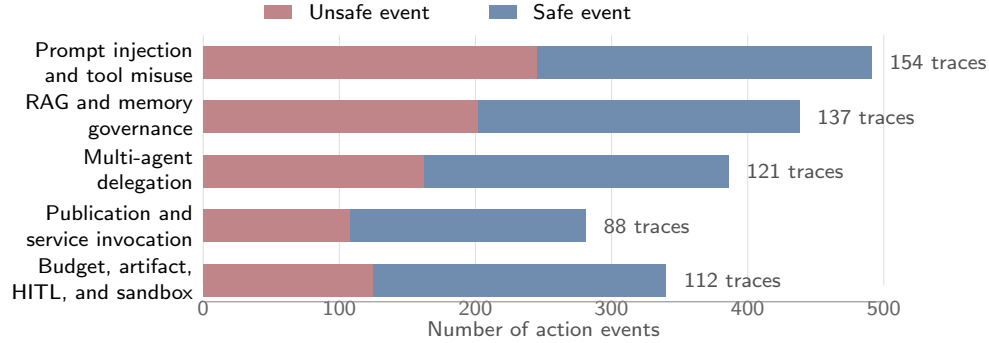
\begin{figure}[H]
\centering
\begin{tikzpicture}[x=0.018cm,y=0.72cm,font=\sffamily\scriptsize]
  \foreach \x in {0,100,200,300,400,500} {
    \draw[black!12] (\x,0.45) -- (\x,5.05);
    \node[anchor=north,text=black!60] at (\x,0.35) {\x};
  }

  \node[anchor=east,align=right] at (0,4.55) {Prompt injection\\and tool misuse};
  \fill[rose!72] (0,4.25) rectangle (245,4.85);
  \fill[deepblue!68] (245,4.25) rectangle (491,4.85);
  \node[anchor=west,xshift=3pt,text=black!68] at (491,4.55) {154 traces};

  \node[anchor=east,align=right] at (0,3.55) {RAG and memory\\governance};
  \fill[rose!72] (0,3.25) rectangle (202,3.85);
  \fill[deepblue!68] (202,3.25) rectangle (438,3.85);
  \node[anchor=west,xshift=3pt,text=black!68] at (438,3.55) {137 traces};

  \node[anchor=east,align=right] at (0,2.55) {Multi-agent\\delegation};
  \fill[rose!72] (0,2.25) rectangle (162,2.85);
  \fill[deepblue!68] (162,2.25) rectangle (386,2.85);
  \node[anchor=west,xshift=3pt,text=black!68] at (386,2.55) {121 traces};

  \node[anchor=east,align=right] at (0,1.55) {Publication and\\service invocation};
  \fill[rose!72] (0,1.25) rectangle (108,1.85);
  \fill[deepblue!68] (108,1.25) rectangle (281,1.85);
  \node[anchor=west,xshift=3pt,text=black!68] at (281,1.55) {88 traces};

  \node[anchor=east,align=right] at (0,0.55) {Budget, artifact,\\HITL, and sandbox};
  \fill[rose!72] (0,0.25) rectangle (125,0.85);
  \fill[deepblue!68] (125,0.25) rectangle (340,0.85);
  \node[anchor=west,xshift=3pt,text=black!68] at (340,0.55) {112 traces};

  \draw[black!45] (0,0.15) -- (500,0.15);
  \node[text=black!68] at (250,-0.35) {Number of action events};
  \fill[rose!72] (45,5.35) rectangle (65,5.60);
  \node[anchor=west] at (70,5.475) {Unsafe event};
  \fill[deepblue!68] (210,5.35) rectangle (230,5.60);
  \node[anchor=west] at (235,5.475) {Safe event};
\end{tikzpicture}
\caption{Composition of the runtime-evaluation workload. Each horizontal bar partitions the action events in a workload family into unsafe and safe events; the annotation at the right reports the number of multi-step traces from which those events were obtained.}
\label{fig:runtime-workload-composition}
\end{figure}

\begin{table}[!htbp]
\centering
\caption{Runtime-evaluation workload.}
\label{tab:runtime-workload}
\small
\begin{tabularx}{\textwidth}{@{}Xrrrr@{}}
\toprule
Workload family & Traces & Events & Unsafe events & Safe events \\
\midrule
Prompt-injection and tool-misuse tasks & 154 & 491 & 245 & 246 \\
RAG and memory-governance tasks & 137 & 438 & 202 & 236 \\
Multi-agent delegation tasks & 121 & 386 & 162 & 224 \\
Publication and service-invocation tasks & 88 & 281 & 108 & 173 \\
Budget, artifact-materialization, HITL, and sandbox tasks & 112 & 340 & 125 & 215 \\
\midrule
Total & 612 & 1,936 & 842 & 1,094 \\
\bottomrule
\end{tabularx}
\tablesource{Authors' Calculations}
\end{table}

Figure~\ref{fig:runtime-enforcement-tradeoffs} complements Table~\ref{tab:runtime-results} by separating percentage outcomes, residual failure counts, and latency. This separation avoids placing heterogeneous units on one axis and makes the security--utility--overhead trade-off visible.

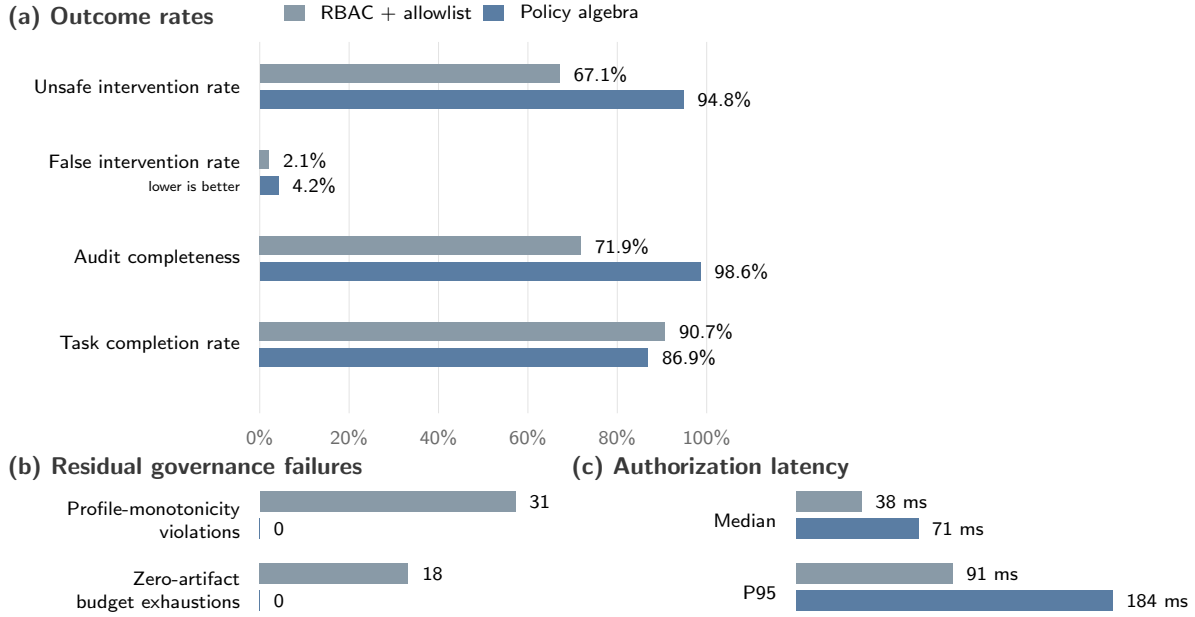
\begin{figure}[H]
\centering
\resizebox{0.96\textwidth}{!}{%
\begin{tikzpicture}[font=\sffamily\scriptsize]
  \node[anchor=west,font=\sffamily\bfseries\small,text=black!78] at (0,8.8) {(a) Outcome rates};
  \foreach \x in {0,20,40,60,80,100} {
    \draw[black!11] ({3.8+0.065*\x},3.05) -- ({3.8+0.065*\x},8.45);
    \node[anchor=north,text=black!58] at ({3.8+0.065*\x},2.95) {\x\%};
  }

  \node[anchor=east] at (3.65,7.80) {Unsafe intervention rate};
  \fill[bluegray!72] (3.8,7.85) rectangle ({3.8+0.065*67.1},8.12);
  \fill[deepblue!78] (3.8,7.48) rectangle ({3.8+0.065*94.8},7.75);
  \node[anchor=west,xshift=2pt] at ({3.8+0.065*67.1},7.985) {67.1\%};
  \node[anchor=west,xshift=2pt] at ({3.8+0.065*94.8},7.615) {94.8\%};

  \node[anchor=east,align=right] at (3.65,6.55) {False intervention rate\\{\tiny lower is better}};
  \fill[bluegray!72] (3.8,6.60) rectangle ({3.8+0.065*2.1},6.87);
  \fill[deepblue!78] (3.8,6.23) rectangle ({3.8+0.065*4.2},6.50);
  \node[anchor=west,xshift=2pt] at ({3.8+0.065*2.1},6.735) {2.1\%};
  \node[anchor=west,xshift=2pt] at ({3.8+0.065*4.2},6.365) {4.2\%};

  \node[anchor=east] at (3.65,5.30) {Audit completeness};
  \fill[bluegray!72] (3.8,5.35) rectangle ({3.8+0.065*71.9},5.62);
  \fill[deepblue!78] (3.8,4.98) rectangle ({3.8+0.065*98.6},5.25);
  \node[anchor=west,xshift=2pt] at ({3.8+0.065*71.9},5.485) {71.9\%};
  \node[anchor=west,xshift=2pt] at ({3.8+0.065*98.6},5.115) {98.6\%};

  \node[anchor=east] at (3.65,4.05) {Task completion rate};
  \fill[bluegray!72] (3.8,4.10) rectangle ({3.8+0.065*90.7},4.37);
  \fill[deepblue!78] (3.8,3.73) rectangle ({3.8+0.065*86.9},4.00);
  \node[anchor=west,xshift=2pt] at ({3.8+0.065*90.7},4.235) {90.7\%};
  \node[anchor=west,xshift=2pt] at ({3.8+0.065*86.9},3.865) {86.9\%};

  \fill[bluegray!72] (4.15,8.75) rectangle (4.45,8.98);
  \node[anchor=west] at (4.55,8.865) {RBAC + allowlist};
  \fill[deepblue!78] (7.05,8.75) rectangle (7.35,8.98);
  \node[anchor=west] at (7.45,8.865) {Policy algebra};

  \node[anchor=west,font=\sffamily\bfseries\small,text=black!78] at (0,2.25) {(b) Residual governance failures};
  \node[anchor=east,align=right] at (3.65,1.50) {Profile-monotonicity\\violations};
  \fill[bluegray!72] (3.8,1.62) rectangle ({3.8+0.12*31},1.92);
  \fill[deepblue!78] (3.8,1.22) rectangle (3.8,1.52);
  \node[anchor=west,xshift=2pt] at ({3.8+0.12*31},1.77) {31};
  \node[anchor=west,xshift=2pt] at (3.8,1.37) {0};
  \node[anchor=east,align=right] at (3.65,0.45) {Zero-artifact\\budget exhaustions};
  \fill[bluegray!72] (3.8,0.57) rectangle ({3.8+0.12*18},0.87);
  \fill[deepblue!78] (3.8,0.17) rectangle (3.8,0.47);
  \node[anchor=west,xshift=2pt] at ({3.8+0.12*18},0.72) {18};
  \node[anchor=west,xshift=2pt] at (3.8,0.32) {0};

  \node[anchor=west,font=\sffamily\bfseries\small,text=black!78] at (8.2,2.25) {(c) Authorization latency};
  \node[anchor=east] at (11.45,1.50) {Median};
  \fill[bluegray!72] (11.6,1.62) rectangle ({11.6+0.025*38},1.92);
  \fill[deepblue!78] (11.6,1.22) rectangle ({11.6+0.025*71},1.52);
  \node[anchor=west,xshift=2pt] at ({11.6+0.025*38},1.77) {38 ms};
  \node[anchor=west,xshift=2pt] at ({11.6+0.025*71},1.37) {71 ms};
  \node[anchor=east] at (11.45,0.45) {P95};
  \fill[bluegray!72] (11.6,0.57) rectangle ({11.6+0.025*91},0.87);
  \fill[deepblue!78] (11.6,0.17) rectangle ({11.6+0.025*184},0.47);
  \node[anchor=west,xshift=2pt] at ({11.6+0.025*91},0.72) {91 ms};
  \node[anchor=west,xshift=2pt] at ({11.6+0.025*184},0.32) {184 ms};
\end{tikzpicture}%
}
\caption{Security, utility, and runtime trade-offs between the RBAC-and-allowlist comparator and the policy-algebra runtime. Panel (a) compares percentage outcomes, Panel (b) shows the two residual governance-failure counts, and Panel (c) reports median and P95 authorization latency. Separate panels preserve the units and directionality of the underlying measures.}
\label{fig:runtime-enforcement-tradeoffs}
\end{figure}

\begin{table}[!htbp]
\centering
\caption{Runtime-enforcement results.}
\label{tab:runtime-results}
\small
\begin{tabularx}{\textwidth}{@{}XrrrX@{}}
\toprule
Metric & RBAC + allowlist & Policy algebra & Difference & Target direction \\
\midrule
Unsafe intervention rate & 67.1\% & 94.8\% & +27.7 pp & Higher is better \\
False intervention rate & 2.1\% & 4.2\% & +2.1 pp & Lower is better \\
Profile monotonicity violations & 31/188 & 0/188 & -31 & Lower is better \\
Zero-artifact budget exhaustions & 18/44 & 0/44 & -18 & Lower is better \\
Audit completeness & 71.9\% & 98.6\% & +26.7 pp & Higher is better \\
Task completion rate & 90.7\% & 86.9\% & -3.8 pp & Higher is better \\
Median gate latency & 38 ms & 71 ms & +33 ms & Lower is better \\
P95 gate latency & 91 ms & 184 ms & +93 ms & Lower is better \\
\bottomrule
\end{tabularx}
\tablesource{Authors' Calculations}
\end{table}

Figure~\ref{fig:scenario-blocking-profile} visualizes the scenario-level results before the exact counts in Table~\ref{tab:scenario-results}. The residual segment makes the remaining attack surface immediately visible, especially for sandbox or network escape attempts.

\begin{figure}[H]
\centering
\begin{tikzpicture}[x=0.070cm,y=0.66cm,font=\sffamily\scriptsize]
  \foreach \x in {0,20,40,60,80,100} {
    \draw[black!11] (\x,0.35) -- (\x,9.45);
    \node[anchor=north,text=black!58] at (\x,0.23) {\x\%};
  }

  \node[anchor=east] at (0,9.0) {Prompt injection to blocked tool};
  \fill[deepblue!76] (0,8.72) rectangle (94.7,9.28);
  \fill[rose!72] (94.7,8.72) rectangle (100,9.28);
  \node[anchor=west,xshift=3pt] at (100,9.0) {94.7\% (126/133)};

  \node[anchor=east] at (0,8.0) {Unauthorized tool execution};
  \fill[deepblue!76] (0,7.72) rectangle (92.9,8.28);
  \fill[rose!72] (92.9,7.72) rectangle (100,8.28);
  \node[anchor=west,xshift=3pt] at (100,8.0) {92.9\% (104/112)};

  \node[anchor=east] at (0,7.0) {RAG over-access};
  \fill[deepblue!76] (0,6.72) rectangle (95.7,7.28);
  \fill[rose!72] (95.7,6.72) rectangle (100,7.28);
  \node[anchor=west,xshift=3pt] at (100,7.0) {95.7\% (111/116)};

  \node[anchor=east] at (0,6.0) {Memory poisoning write};
  \fill[deepblue!76] (0,5.72) rectangle (91.9,6.28);
  \fill[rose!72] (91.9,5.72) rectangle (100,6.28);
  \node[anchor=west,xshift=3pt] at (100,6.0) {91.9\% (79/86)};

  \node[anchor=east] at (0,5.0) {Profile laundering};
  \fill[deepblue!76] (0,4.72) rectangle (100,5.28);
  \node[anchor=west,xshift=3pt] at (100,5.0) {100.0\% (104/104)};

  \node[anchor=east] at (0,4.0) {External side effect without HITL};
  \fill[deepblue!76] (0,3.72) rectangle (94.5,4.28);
  \fill[rose!72] (94.5,3.72) rectangle (100,4.28);
  \node[anchor=west,xshift=3pt] at (100,4.0) {94.5\% (121/128)};

  \node[anchor=east] at (0,3.0) {Publication/profile drift};
  \fill[deepblue!76] (0,2.72) rectangle (98.4,3.28);
  \fill[rose!72] (98.4,2.72) rectangle (100,3.28);
  \node[anchor=west,xshift=3pt] at (100,3.0) {98.4\% (61/62)};

  \node[anchor=east] at (0,2.0) {Budget or rate exhaustion};
  \fill[deepblue!76] (0,1.72) rectangle (100,2.28);
  \node[anchor=west,xshift=3pt] at (100,2.0) {100.0\% (58/58)};

  \node[anchor=east] at (0,1.0) {Sandbox/network escape};
  \fill[deepblue!76] (0,0.72) rectangle (79.1,1.28);
  \fill[rose!72] (79.1,0.72) rectangle (100,1.28);
  \node[anchor=west,xshift=3pt] at (100,1.0) {79.1\% (34/43)};

  \draw[black!45] (0,0.58) -- (100,0.58);
  \node[text=black!68] at (50,-0.75) {Share of unsafe events};
  \fill[deepblue!76] (14,9.78) rectangle (18,10.08);
  \node[anchor=west] at (20,9.93) {Intervened};
  \fill[rose!72] (56,9.78) rectangle (60,10.08);
  \node[anchor=west] at (62,9.93) {Residual miss};
\end{tikzpicture}
\captionsetup{skip=0.8em}
\caption{Unsafe-action intervention profile by scenario. Blue segments show the share of unsafe events denied, escalated, restricted, or redirected by the proposed runtime, while red segments show residual unrestricted allows. Parentheses report the intervened and unsafe event counts used to compute each rate.}
\label{fig:scenario-blocking-profile}
\end{figure}
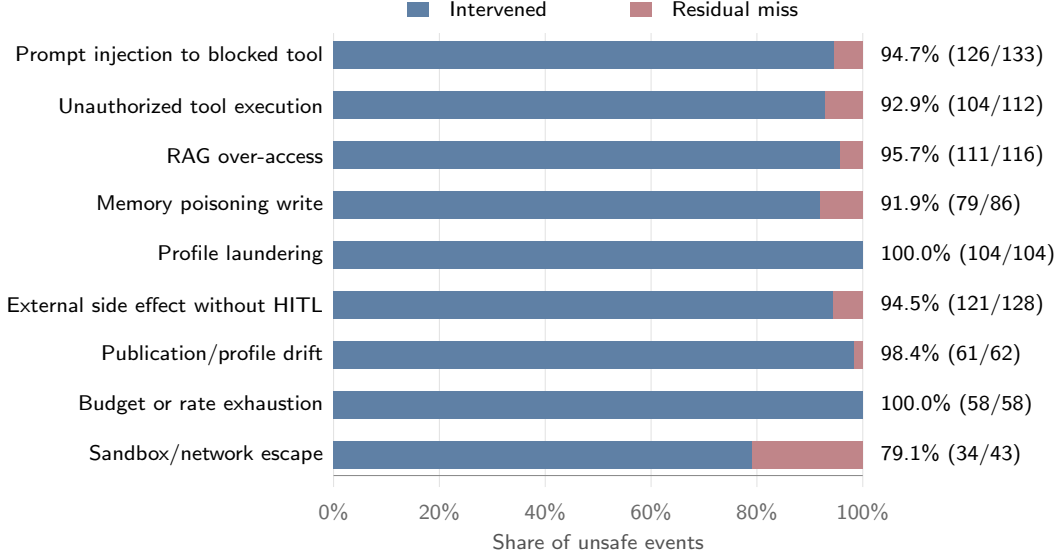

\begin{table}[!htbp]
\centering
\caption{Unsafe-action intervention by scenario type for the proposed runtime.}
\label{tab:scenario-results}
\small
\begin{tabularx}{\textwidth}{@{}XrrrX@{}}
\toprule
Scenario & Unsafe events & Intervened & Intervention rate & Main residual failure mode \\
\midrule
Prompt injection triggers blocked tool & 133 & 126 & 94.7\% & Ambiguous tool intent \\
Unauthorized tool execution & 112 & 104 & 92.9\% & Incomplete tool metadata \\
RAG over-access & 116 & 111 & 95.7\% & Misclassified document label \\
Memory poisoning write & 86 & 79 & 91.9\% & Semantic poisoning not caught by exact filter \\
Profile laundering through sub-agent & 104 & 104 & 100.0\% & None observed \\
External side effect without HITL & 128 & 121 & 94.5\% & Missing action-risk tag \\
Publication/profile drift & 62 & 61 & 98.4\% & Stale service metadata \\
Budget or rate exhaustion & 58 & 58 & 100.0\% & None observed \\
Sandbox or network escape attempt & 43 & 34 & 79.1\% & Incomplete command classification \\
\midrule
Total & 842 & 798 & 94.8\% & -- \\
\bottomrule
\end{tabularx}
\tablesource{Authors' Calculations}
\end{table}

In the evaluation run, the policy-algebra runtime intervenes on 798 of 842 unsafe events, giving an unsafe intervention rate of 94.8\%. This is higher than the RBAC-and-allowlist comparator, which intervenes on 565 of 842 unsafe events. For the proposed runtime, the 798 interventions comprise 619 denials, 121 required-approval decisions, and 58 artifact-materialization redirects. The distinction matters: the reported intervention rate measures prevention of immediate unrestricted execution, not permanent rejection of every action. The largest gains appear in cases where the comparator authorizes the user or tool in isolation but does not evaluate profile, data class, service-publication state, budget, artifact state, HITL, or delegation lineage.

The improved intervention rate has operational cost. Relative to the comparator, the policy-algebra runtime increases the false intervention rate from 2.1\% to 4.2\%, and median authorization latency increases from 38 ms to 71 ms per gated action. In exchange, profile monotonicity violations fall from 31 observed violations under the comparator to zero under the proposed runtime. Zero-artifact budget exhaustions fall from 18 observed runs under budget-only enforcement to zero under cost-aware materialization. Audit completeness increases from 71.9\% to 98.6\%, which makes failed cases easier to reproduce and repair.

The paired counts make the reliability--capability exchange more concrete. On the same workload, the proposed runtime produces 233 additional interventions on unsafe events (798 rather than 565) and 23 additional interventions on safe events (46 rather than 23). Task completion decreases by 3.8 percentage points, from 90.7\% to 86.9\%; equivalently, the governed runtime retains 95.8\% of the comparator's completion rate on this workload. These are descriptive workload-specific comparisons, not causal estimates or universal exchange rates. They nevertheless show the intended scientific object: not maximum unconstrained completion, but conversion of raw capability into reliable capability under an explicit policy envelope.

For reproducibility, each measured run should export event-level rows with the following fields: trace identifier, event identifier, scenario family, ground-truth label, runtime decision, failed predicate, caller profile, effective profile, tool or service identifier, data class, HITL requirement, HITL approval state, budget before and after the event, artifact state, materialization decision, audit-completeness flag, and authorization latency in milliseconds.

The results are consistent with the main claim of the paper: agent security should be treated as a trace property rather than as a property of the final answer alone. A final answer can look correct while the trace violates policy. A reliable runtime therefore must govern the reasoning-to-action transition. The relevant question is not whether a model is safe in isolation, but whether each action event satisfies the policy predicates in \eqref{eq:reliable-predicate} and the feasibility condition in \eqref{eq:feasibility-formal}. Because the evaluation is observational over a finite conformance workload, it demonstrates measured policy conformance and trade-offs; it does not establish universal security or superiority on general capability benchmarks.

The largest observed improvement comes from decisions that require more than one authorization source. The RBAC-and-allowlist comparator misses cases where the user or tool is permitted in isolation but the action still violates profile, data, publication, budget, artifact state, HITL, or delegation constraints. This supports the policy-conjunction rule in \eqref{eq:permit}. A tool allowlist is too weak because it can ignore data labels, service publication, budget, artifact state, HITL, or call-chain state. RBAC is too weak if it only checks a human user. The conjunction treats tool execution as a joint decision across several policy sources and makes denials easier to explain because each denial can be traced to a failed factor.

The zero observed profile-monotonicity violations under the policy-algebra runtime provide finite-workload evidence consistent with the trust-preservation property of the delegation model. Delegation is often treated as a functional design choice. In this model, it is a security event. The caller's trust context must pass to the callee. A lower-security callee can still be used, but only under the stricter effective profile and narrowed authority. This rule does not solve all multi-agent risk; rather, under correct context construction and atomic enforcement, it rules out the specific privilege-bypass mechanism in which delegation weakens the caller's security state.

The audit-completeness result is also important. Audit quality is not merely a logging feature; it is part of the feasibility of governed execution. When the runtime records identity, profile, policy decision, data class, approval state, budget state, artifact state, call-chain lineage, and result hash, failed cases become easier to reproduce, review, and repair. This connects the runtime measurements directly to the policy-repair loop in Section~\ref{sec:policyrepair}.

The artifact-materialization result adds an economic reliability dimension. A budget-only runtime can stop overspend but still allow the harmful path in which the entire budget is consumed with no recoverable output. Algorithm~\ref{alg:materialization} changes the admissible action set as cost exposure rises. If no artifact exists, the run is redirected toward a draft, checkpoint, or other durable output before the remaining budget is exhausted. The zero observed exhaustion cases establish conformance for the tested materialization conditions, not a guarantee that every produced artifact has adequate semantic quality. Equation~\eqref{eq:valid-artifact} states the stronger validity target, and Eq.~\eqref{eq:materialization-reserve} identifies the assumptions needed to preserve capacity for a materialization attempt.

The RAG and memory equations are central to this interpretation. Many agent failures arise when untrusted data enters context or when a memory write affects future behavior. Equation~\eqref{eq:read-admissible} makes retrieval a policy decision over classification, access control, scope, and purpose. Equation~\eqref{eq:write-admissible} makes memory writes subject to classification, ownership, scope, retention, and sanitization. The scenario-level results show why these controls need to be evaluated separately from ordinary tool authorization.

The residual misses in Table~\ref{tab:scenario-results} also locate the empirical boundary of the formal guarantees. Ambiguous intent, incomplete tool metadata, misclassified documents, semantic poisoning, missing risk tags, stale service metadata, and incomplete command classification are all errors in constructing or interpreting the presented context $\widehat c_t$. They therefore instantiate the first term of the conditional-correctness decomposition in Eq.~\eqref{eq:conditional-correctness}, rather than refuting the algebraic result for correctly represented contexts. Scientifically, this distinction separates two research problems: preserving policy under composition, which the algebra addresses, and accurately perceiving the security-relevant context, which requires better metadata, classifiers, provenance, and attestation.

Publication-time profile freezing connects runtime safety with lifecycle governance. A service published through MCP or REST becomes a capability. If its security profile changes silently after approval, a consumer may invoke a different risk state than the one reviewed. Including $p_{\mathrm{publication}}$ in the profile cascade preserves the approved security state at runtime. A change can still happen, but it becomes a review event.

The improved intervention rate has a visible trade-off: false interventions and latency increase. This is expected because the policy-algebra runtime evaluates more predicates before permitting execution. These costs are part of the optimization in \eqref{eq:objective}: the runtime seeks task completion while minimizing risk and cost under hard policy constraints. The result is not a claim that stricter policy is free; it is evidence that additional enforcement exchanges a limited amount of raw capability for profile monotonicity, bounded-loss artifact production, stronger unsafe-action intervention, and more complete audit evidence.

The residual-risk model should therefore be interpreted as a calibration layer rather than the foundation of safety. The platform is not assumed to perfectly estimate risk for every action. Instead, hard constraints define the non-negotiable safety boundary, and the decomposed risk model ranks feasible alternatives, assigns review thresholds, and guides control investment. This avoids placing too much confidence in a single risk score while still creating an empirical path for improvement through telemetry and red-team results.

The method is most useful where agents cross organizational boundaries: enterprise assistants with connectors, multi-agent workflows, governed RAG systems, published MCP and REST services, finance and procurement agents, audit automation, legal review, and compliance workflows. These domains require explicit evidence that actions were authorized, bounded, recoverable, approved, and auditable. The policy algebra gives such evidence a formal structure.

\section{Policy Repair and Regression Testing}
\label{sec:policyrepair}

Policy repair closes the loop between evaluation and enforcement. A failed trace becomes a new constraint or a new regression case.

Evaluation should change the policy set when a failure is found. Let $\Policies_t$ be the policy set before evaluation and let $\mathcal{F}_t$ be the failed cases. A repair function converts failures into policy updates:
\begin{equation}
    \Policies_{t+1}=\Policies_t\cup \mathrm{Repair}(\mathcal{F}_t).
    \label{eq:policyrepair}
\end{equation}
where $\Policies_t$ is the policy set before evaluation round $t$, $\mathcal{F}_t$ is the set of failed cases, and $\mathrm{Repair}(\mathcal{F}_t)$ is the set of new or strengthened policy rules induced by those failures.
The regression suite also grows:
\begin{equation}
    \mathcal{S}_{t+1}=\mathcal{S}_t\cup \mathcal{F}_t.
    \label{eq:regressionupdate}
\end{equation}
where $\mathcal{S}_t$ is the regression suite before repair and $\mathcal{S}_{t+1}$ is the suite after adding failed cases as future tests.
This makes evaluation part of the control loop. A failed test becomes a profile update, tool restriction, memory rule, publication rule, materialization threshold update, HITL rule, audit requirement, or regression case.

\begin{lstlisting}[language=Python,caption={Policy repair sketch.}]
def repair_policy(policy, failed_cases):
    for case in failed_cases:
        if case.kind == "unsafe_tool_allowed":
            policy.tools[case.tool].min_profile = max(
                policy.tools[case.tool].min_profile,
                case.required_profile,
            )
            policy.tools[case.tool].hitl_required = True
        elif case.kind == "profile_downgrade":
            policy.delegation.enforce_profile_join = True
            policy.delegation.max_depth = min(policy.delegation.max_depth, case.max_depth)
        elif case.kind == "rag_access_violation":
            policy.rag.block_labels.add(case.data_label)
            policy.rag.require_scope_match = True
        elif case.kind == "zero_artifact_exhaustion":
            policy.materialization.lower_threshold(case.workflow_type)
            policy.materialization.require_checkpoint(case.workflow_type)
        elif case.kind == "missing_audit_field":
            policy.audit.required_fields.add(case.field)
        elif case.kind == "publication_drift":
            policy.publication.require_review_on_profile_change = True
    return policy
\end{lstlisting}

\begin{longtable}{@{}>{\raggedright\arraybackslash}p{0.04\textwidth}
>{\raggedright\arraybackslash}p{0.23\textwidth}
>{\raggedright\arraybackslash}p{0.31\textwidth}
>{\raggedright\arraybackslash}p{0.30\textwidth}@{}}
\caption{Policy repair actions after failed evaluation cases.}\label{tab:policyrepair}\\
\toprule
No. & Failed case & Repair action & Regression case \\
\midrule
\endfirsthead
\toprule
No. & Failed case & Repair action & Regression case \\
\midrule
\endhead
1 & Unsafe tool allowed & Add tool profile floor and HITL rule & Re-run blocked tool-call test \\
2 & Profile downgrade in delegation & Enforce profile join on call chain & Re-run delegation test \\
3 & Unauthorized document retrieved & Restrict RAG source and add classification check & Re-run RAG access test \\
4 & Missing audit field & Add required field to audit schema & Re-run audit completeness test \\
5 & Published service changed profile & Require re-approval before invocation & Re-run publication-freeze test \\
6 & Budget exhausted without artifact & Lower materialization threshold and require checkpoint & Re-run artifact-materialization test \\
\bottomrule
\end{longtable}
\tablesource{Authors' formulation}

\section{Limitations and Future Work}
\label{sec:limits}

The policy algebra is conservative by construction, but conservatism is not the same as complete security. Its theorems are conditional guarantees over represented policies, contexts, and enforcement assumptions; its experiments are finite-workload measurements.

The formulation is a policy algebra, not a complete proof of system security. The algebra establishes trust preservation and least-restrictive sound composition when the governing policies and security context are represented correctly. It does not prove that the runtime has perceived the true context, that every relevant predicate has been specified, or that an implementation cannot bypass the gate. Additional formal work is needed for non-escalation across dynamic call graphs, completeness of audit reconstruction, and preservation of monotonicity under policy updates.

The quality of the runtime depends on correct metadata for data objects, services, tools, memory entries, and published endpoints. If labels, scopes, provenance, or tool criticality are wrong, policy predicates may make wrong decisions. This limitation is captured explicitly by the context-error term in Eq.~\eqref{eq:conditional-correctness}. Future work should include metadata attestation, signed tool registries, cryptographic service attestations, and continuous validation of data classifications.

The profile set used in this paper is intentionally simple. Real organizations may require richer partial orders with domain-specific incomparable profiles. For example, a safety-critical profile and a privacy-critical profile may not be linearly ordered. Future work should study richer lattices, type systems, and policy simulation before deployment.

The conformance procedure assumes that many policy predicates are available and sufficiently deterministic. Some predicates may rely on classifiers for prompt injection, data sensitivity, provenance quality, or output safety. Such classifiers can fail. A production runtime should treat classifier outputs as evidence with uncertainty rather than as perfect truth.

The multi-agent extension handles call chains and delegation edges. The budget non-amplification result additionally assumes atomic reservations at fan-out. More complex multi-agent systems may include concurrent branching, negotiation, broadcast communication, shared memory, or cyclic interaction. These settings need graph-level trust analysis, transactional budget allocation, and possibly fixed-point semantics.

The zero-artifact metric used in the current evaluation establishes that a durable output state was reached before exhaustion under the tested conditions; it does not fully test the nontriviality, schema, provenance, and resumability clauses of Eq.~\eqref{eq:valid-artifact}. Artifact-quality scoring and workflow-specific validity tests are therefore required before claiming semantic recoverability. Other implementation elements remain partial, including deterministic tooling for all regulated workflows, asynchronous human-approval queues at scale, immutable audit storage, and advanced cross-session memory controls.

Finally, the reported measurements demonstrate conformance and a reliability--capability trade-off on the stated workload, not universal attack resistance or benchmark superiority. Further evaluation should include repeated runs, uncertainty intervals, larger and independently constructed conformance workloads, internal red-team scenarios, enterprise traces, concurrent delegation graphs, and service-publication workflows. External datasets may be added for comparative context, but they are not the normative standard for this evaluation. Metrics should include jointly reported task success, unsafe and false intervention rates, attack success, latency, cost, valid-artifact exhaustion rate, audit completeness, and profile-monotonicity violations.

\section{Conclusion}
\label{sec:conclusion}

The paper developed a policy algebra for converting raw agent capability into reliable capability. The motivating question was how an enterprise agent platform can preserve useful autonomy while ensuring that each executed action remains authorized, non-amplifying under delegation, economically recoverable, and auditable. The answer is a two-stage method: first, compile standards, threat models, and governance artifacts into executable runtime predicates; second, compose those predicates into a reliability envelope using algebraic operators that tighten profiles, narrow authority, reserve budget, accumulate materialization and approval obligations, and preserve audit evidence.

The method treats agentic reliability as constrained action selection over execution traces. An agent is reliably capable only when it reaches the task goal through a trace in which every governed action satisfies identity, role, profile, data, memory, tool, budget, valid-artifact, HITL, and audit constraints. The formal contribution is two-sided: composition is trust-preserving, so governing obligations cannot be weakened, and least restrictive, so the composed state does not remove authority beyond what those obligations require. Profiles are ordered policy objects; tools are authorized through intersected predicates; service publication is frozen under reviewed profile state; and multi-agent calls propagate trust through profile joins, tool intersections, budget conservation, memory-scope narrowing, and approval inheritance.

The main contribution is not another risk list or a claim of unrestricted benchmark superiority. It is a correctness-oriented method for converting risk lists and security standards into executable runtime behavior while retaining useful task capability. In the reported workload, stronger unsafe-event intervention, zero observed profile and zero-artifact violations, and more complete audit evidence are obtained with higher false intervention and latency and a 3.8-percentage-point reduction in task completion. This measured trade-off, together with the conditional guarantees and explicit assumptions, provides a falsifiable path for improving both sides of the objective: expand the reliability envelope through better context perception and policy calibration without weakening its invariants. That is the practical route from agents that can act to agents that can act reliably.

\section*{Declarations}

\paragraph{Funding}
The authors received no funding for the submitted work from any organization.

\paragraph{Conflict of Interest}
The authors have no relevant financial or non-financial interests to disclose.

\paragraph{Ethical Approval}
This article does not contain any studies with human participants or animals performed by any of the authors.

\paragraph{Data Availability}
The datasets generated and/or analysed during the current study are available in the GitHub repository: \url{https://github.com/bhaskatripathi/PolicyAlgebra}.

\appendix
\section*{Appendix}
\addcontentsline{toc}{section}{Appendix}
\renewcommand{\thesection}{A\arabic{section}}
\renewcommand{\thetable}{A\arabic{table}}
\setcounter{section}{0}
\setcounter{table}{0}

\section{Notation Summary}

\begin{table}[H]
\centering
\caption{Summary of notation used in the paper and appendices.}
\label{tab:notation-summary}
\small
\begin{tabularx}{\textwidth}{@{}p{0.27\textwidth}X@{}}
\toprule
Symbol & Meaning \\
\midrule
$G$ & Set of agents. \\
$T$ & Set of tools. \\
$D$ & Set of data sources or data objects. \\
$V$ & Set of published services. \\
$\Profiles$ & Set of security profiles. \\
$p\vee q$ or $p\join q$ & Profile join, equal to the stricter profile under the profile order. \\
$M_p$ & Allowed model set under profile $p$. \\
$T_p$ & Allowed tool set under profile $p$. \\
$K_p$ & Memory and RAG policy under profile $p$. \\
$B_p$ & Budget under profile $p$. \\
$Q_p$ & Artifact-materialization policy under profile $p$. \\
$O_t$ & Durable artifact or checkpoint state at step $t$. \\
$H_p$ & Human-in-the-loop rule under profile $p$. \\
$E_p$ & External communication policy under profile $p$. \\
$S_p$ & Schema and deterministic execution policy under profile $p$. \\
$A_p$ & Audit depth under profile $p$. \\
$\Gamma$ & Compiler or mapping from assessment artifacts to runtime controls; $\Gamma(p)$ denotes the policy object induced by profile $p$. \\
$\pi_{\mathrm{ctrl}}$ & Runtime control predicate. \\
$\xi$ or $\rho_{\mathrm{trace}}$ & Execution trace. \\
$\Xi(g,\zeta)$ & Set of traces that agent or service $g$ can generate for task $\zeta$. \\
$\mathsf{Capable}(g,\zeta)$ & Existence of a goal-reaching trace, without a path-admissibility requirement. \\
$\mathsf{ReliablyCapable}(g,\zeta,\varpi_{0:T})$ & Existence of a goal-reaching trace whose action events all satisfy the governing policy sequence. \\
$C_t$ & Cumulative consumed cost at step $t$. \\
$\mathsf{Artifact}(\xi_t)$ & Predicate indicating whether a trace prefix has produced a durable recoverable artifact. \\
$\mathsf{ValidArtifact}(o,Q)$ & Workflow-specific conjunction of durability, nontriviality, schema validity, provenance, and resumability checks. \\
$R_Q(s_t)$ & Reserved capacity required to attempt artifact materialization from state $s_t$ under policy $Q$. \\
$\mathsf{Materialize}$ & Predicate enforcing cost-aware artifact materialization. \\
\bottomrule
\end{tabularx}
\end{table}

\section{Security Profile Specification}

\begin{table}[H]
\centering
\caption{Profile dimensions and representative runtime interpretation.}
\label{tab:appendix-profile-specification}
\scriptsize
\begin{tabularx}{\textwidth}{@{}p{0.13\textwidth}XXXX@{}}
\toprule
Dimension & Low & Medium & High & Very High \\
\midrule
Models & Approved general models & Approved models with logging & Restricted models & Restricted models; deterministic tools for regulated action \\
Tools & Built-in low-risk tools & Approved tools & Tool calls require stricter validation & Deterministic or schema-bound tools \\
Memory & Short-lived & Scoped with retention & Scoped, redacted, audited & Minimal retention; strong isolation \\
RAG & Public or low-risk sources & Approved internal sources & Classified sources with provenance & Sanitized sources only; stronger review \\
HITL & None by default & Destructive actions & External side effects and high-risk writes & Most external actions \\
Budget & Basic run limit & User/team budget & Strict budget and rate limit & Fixed budget and approval for increase \\
Artifact materialization & Best-effort output & Checkpoint when budget pressure rises & Durable draft before high-cost continuation & Artifact-producing actions only near exhaustion \\
Audit & Basic run log & Tool-level log & Step-level trace & High detail and immutable storage \\
External communication & Restricted & Approved connectors & Approved connectors with review & Approved service paths only \\
Publication & Not allowed & Team/internal & Reviewed publication & Reviewed and frozen with profile cap \\
\bottomrule
\end{tabularx}
\end{table}

\section{Detailed Mapping from Standards to Runtime Effects}

\begin{table}[H]
\centering
\caption{Detailed crosswalk from external security artifacts to executable runtime effects and evidence.}
\label{tab:appendix-standards-runtime-effects}
\scriptsize
\begin{tabularx}{\textwidth}{@{}p{0.04\textwidth}p{0.12\textwidth}p{0.19\textwidth}X p{0.16\textwidth}@{}}
\toprule
No. & Artifact & Example category & Runtime effect & Evidence \\
\midrule
1 & AISVS & Access control & RBAC predicate, data-source access check, service invocation check & Role, resource, decision \\
2 & AISVS & Agentic action security & Tool authorization, HITL, sandbox, schema validation & Tool, parameters, decision \\
3 & AISVS & Memory/vector security & $\mathsf{Read}_p$, $\mathsf{Write}_p$, retention, redaction, isolation & Source, memory key, class \\
4 & AISVS & Monitoring & Trace completeness, anomaly detection, alerts & Trace ID, metric \\
5 & NIST & Govern; accountability & Owner, reviewer, profile floor, approval path & Owner, approver \\
6 & NIST & Map; context and assets & Agent inventory, tool list, data classification map & Asset ID \\
7 & NIST & Measure; testing and monitoring & Evaluation set, score, denial rate, cost rate & Metric \\
8 & NIST & Manage; risk treatment & Deny, approve, revoke, raise profile, restrict tool, require artifact & Action taken \\
9 & MAESTRO & Tool layer & Sandbox, allowlist, command filter & Tool evidence \\
10 & MAESTRO & Ecosystem layer & Trust propagation, depth limit & Call chain \\
11 & MITRE ATLAS & Attack technique & Detection rule, mitigation, audit tag & Technique ID \\
12 & EU AI Act/GDPR & Personal data and high-risk use & Data class policy, redaction, documentation, HITL & Class, basis, approval \\
13 & Quantitative risk & Bounded downside & Materialization threshold, artifact-state gate, stop-loss decision & Cost ratio, artifact state, decision \\
\bottomrule
\end{tabularx}
\end{table}

\end{document}